\documentclass[11pt]{article}
\pdfoutput=1

\usepackage[letterpaper, left=1in, right=1in, top=1in, bottom=1in]{geometry}

\usepackage[dvipsnames]{xcolor}
\usepackage[colorlinks=true, linkcolor=blue!70!black, citecolor=blue!70!black,urlcolor=black,breaklinks=true]{hyperref}
\usepackage{microtype}

\usepackage{natbib}
\bibpunct{(}{)}{;}{a}{,}{,}

\usepackage{amsthm}
\usepackage{mathtools}
\usepackage{amsmath}
\usepackage{bbm}
\usepackage{amsfonts}
\usepackage{amssymb}

\usepackage{MnSymbol} 

\usepackage{xpatch}

\theoremstyle{definition}  

\newtheorem{lemma}{Lemma}

\newtheorem{proposition}{Proposition}

\newtheorem{assumption}{Assumption}
\theoremstyle{plain}

\newtheorem{theorem}{Theorem}
\newtheorem{definition}{Definition}

\xpatchcmd{\proof}{\itshape}{\normalfont\proofnameformat}{}{}
\newcommand{\proofnameformat}{\bfseries}

\usepackage{prettyref}
\newcommand{\pref}[1]{\prettyref{#1}}
\newcommand{\pfref}[1]{Proof of \prettyref{#1}}
\newcommand{\savehyperref}[2]{\texorpdfstring{\hyperref[#1]{#2}}{#2}}
\newrefformat{eq}{\savehyperref{#1}{\textup{(\ref*{#1})}}}
\newrefformat{eqn}{\savehyperref{#1}{Equation~\ref*{#1}}}
\newrefformat{lem}{\savehyperref{#1}{Lemma~\ref*{#1}}}
\newrefformat{def}{\savehyperref{#1}{Definition~\ref*{#1}}}
\newrefformat{line}{\savehyperref{#1}{Line~\ref*{#1}}}
\newrefformat{thm}{\savehyperref{#1}{Theorem~\ref*{#1}}}
\newrefformat{corr}{\savehyperref{#1}{Corollary~\ref*{#1}}}
\newrefformat{cor}{\savehyperref{#1}{Corollary~\ref*{#1}}}
\newrefformat{sec}{\savehyperref{#1}{Section~\ref*{#1}}}
\newrefformat{app}{\savehyperref{#1}{Appendix~\ref*{#1}}}
\newrefformat{ass}{\savehyperref{#1}{Assumption~\ref*{#1}}}
\newrefformat{ex}{\savehyperref{#1}{Example~\ref*{#1}}}
\newrefformat{fig}{\savehyperref{#1}{Figure~\ref*{#1}}}
\newrefformat{alg}{\savehyperref{#1}{Algorithm~\ref*{#1}}}
\newrefformat{rem}{\savehyperref{#1}{Remark~\ref*{#1}}}
\newrefformat{conj}{\savehyperref{#1}{Conjecture~\ref*{#1}}}
\newrefformat{prop}{\savehyperref{#1}{Proposition~\ref*{#1}}}
\newrefformat{proto}{\savehyperref{#1}{Protocol~\ref*{#1}}}
\newrefformat{prob}{\savehyperref{#1}{Problem~\ref*{#1}}}
\newrefformat{claim}{\savehyperref{#1}{Claim~\ref*{#1}}}

\DeclarePairedDelimiter{\abs}{\lvert}{\rvert} %
\DeclarePairedDelimiter{\brk}{[}{]}
\DeclarePairedDelimiter{\crl}{\{}{\}}
\DeclarePairedDelimiter{\prn}{(}{)}
\DeclarePairedDelimiter{\nrm}{\|}{\|}
\DeclarePairedDelimiter{\tri}{\langle}{\rangle}

\DeclarePairedDelimiter{\ceil}{\lceil}{\rceil}

\DeclareMathOperator{\En}{\mathbb{E}}

\DeclareMathOperator*{\argmin}{arg\,min} 

\newcommand{\wt}[1]{\widetilde{#1}}
\newcommand{\wh}[1]{\widehat{#1}}

\def\ddefloop#1{\ifx\ddefloop#1\else\ddef{#1}\expandafter\ddefloop\fi}
\def\ddef#1{\expandafter\def\csname bb#1\endcsname{\ensuremath{\mathbb{#1}}}}
\ddefloop ABCDEFGHIJKLMNOPQRSTUVWXYZ\ddefloop
\def\ddefloop#1{\ifx\ddefloop#1\else\ddef{#1}\expandafter\ddefloop\fi}
\def\ddef#1{\expandafter\def\csname b#1\endcsname{\ensuremath{\mathbf{#1}}}}
\ddefloop ABCDEFGHIJKLMNOPQRSTUVWXYZ\ddefloop
\def\ddef#1{\expandafter\def\csname c#1\endcsname{\ensuremath{\mathcal{#1}}}}
\ddefloop ABCDEFGHIJKLMNOPQRSTUVWXYZ\ddefloop
\def\ddef#1{\expandafter\def\csname h#1\endcsname{\ensuremath{\widehat{#1}}}}
\ddefloop ABCDEFGHIJKLMNOPQRSTUVWXYZ\ddefloop
\def\ddef#1{\expandafter\def\csname hc#1\endcsname{\ensuremath{\widehat{\mathcal{#1}}}}}
\ddefloop ABCDEFGHIJKLMNOPQRSTUVWXYZ\ddefloop
\def\ddef#1{\expandafter\def\csname t#1\endcsname{\ensuremath{\widetilde{#1}}}}
\ddefloop ABCDEFGHIJKLMNOPQRSTUVWXYZ\ddefloop
\def\ddef#1{\expandafter\def\csname tc#1\endcsname{\ensuremath{\widetilde{\mathcal{#1}}}}}
\ddefloop ABCDEFGHIJKLMNOPQRSTUVWXYZ\ddefloop

\newcommand{\Holder}{H{\"o}lder}

\newcommand{\ls}{\ell}
\newcommand{\ind}{\mathbbm{1}}    
\newcommand{\pmo}{\crl*{\pm{}1}}
\newcommand{\eps}{\epsilon}
\newcommand{\veps}{\varepsilon}

\newcommand{\ldef}{\vcentcolon=}
\newcommand{\rdef}{=\vcentcolon}

\newcommand{\mb}[1]{\boldsymbol{#1}}

\newcommand{\Bspace}{\mathfrak{B}}

\newcommand{\trn}{\intercal}

\newcommand{\yh}{\hat{y}}

\renewcommand{\trn}{\dagger}

\usepackage{breakcites}

\usepackage{etoolbox}
\newtoggle{icml}
\togglefalse{icml}

\usepackage{microtype}
\usepackage{graphicx}
\usepackage{subfigure}
\usepackage{booktabs} 

\usepackage[scaled=.88]{helvet}
\usepackage{xspace}
\usepackage{setspace}

\usepackage{accents}

\usepackage{enumitem}

\renewcommand{\paragraph}[1]{\par\vspace{10pt}\noindent\textbf{#1}\hspace{5pt}}

\usepackage{algorithm}
\usepackage{verbatim}
\usepackage[noend]{algpseudocode}

\newcommand{\algcomment}[1]{\textcolor{blue!70!black}{\footnotesize{\texttt{\textbf{//
          #1}}}}}

\newtheorem{innercustomassumption}{Assumption}
\newenvironment{customassumption}[1]
  {\renewcommand\theinnercustomassumption{#1}\innercustomassumption}
  {\endinnercustomassumption}

\newcommand{\richcb}{\textsf{RichCB}\xspace}
\newcommand{\richcbs}{\textsf{RichCB}s\xspace}

\newcommand{\midsem}{\,;}

\renewcommand{\trn}{\top}

\newcommand{\psdgeq}{\succeq}

\newcommand{\psdgt}{\succ}

\newcommand{\eigmin}{\lambda_{\mathrm{min}}}

\newcommand{\fhat}{\wh{f}}

\renewcommand{\ind}[1]{^{{\scriptscriptstyle(#1)}}}

\newcommand{\fstar}{f^{\star}}

\newcommand{\pistar}{\pi^{\star}}

\newcommand{\Reg}{\mathrm{Reg}_{\mathsf{CB}}(T)}
\newcommand{\RegSquare}{\mathrm{Reg}_{\mathsf{Sq}}(T)}
\newcommand{\alg}{\textup{\textsf{SqAlg}}\xspace}

\newcommand{\indic}{\mathbbm{1}}

\newcommand{\gstar}{g^{\star}}

\newcommand{\istar}{i^{\star}}

\newcommand{\aone}{a\ind{1}}
\newcommand{\atwo}{a\ind{2}}

\newcommand{\cb}{\mathsf{A}}

\newcommand{\poly}{\mathrm{poly}}
\newcommand{\polylog}{\mathrm{polylog}}

\renewcommand{\yh}{\wh{y}}
\newcommand{\astar}{a^{\star}}

\renewcommand{\midsem}{\,;}
\newcommand{\mainalg}{\textup{\textsf{SquareCB}}\xspace}
\newcommand{\infalg}{\textup{\textsf{SquareCB.Hilbert}}\xspace}

\newcommand{\glmtron}{\textsf{GLMtron}\xspace}
\newcommand{\linucb}{\textsf{LinUCB}\xspace}

\newcommand{\bistro}{\textsf{BISTRO}\xspace}
\newcommand{\iltcb}{\textsf{ILTCB}\xspace}

\newcommand{\iltcblong}{\textsf{ILOVETOCONBANDITS}\xspace}
\newcommand{\egreedy}{$\veps$-\textsf{Greedy}\xspace}

\newcommand{\bigoh}{\cO}
\newcommand{\bigoht}{\wt{\cO}}

\newcommand{\bigomt}{\wt{\Omega}}

\newcommand{\op}{\mathrm{op}}

\newcommand{\Aball}{\bbB_{\infdim}}
\newcommand{\byh}{\wh{\mb{y}}}
\newcommand{\bfstar}{\mb{f}^{\star}}
\newcommand{\ytil}{\wt{\mb{y}}}
\newcommand{\ftil}{\wt{\mb{f}}}

\newcommand{\Hiid}{\cH^{\mathrm{iid}}}

\newcommand{\bls}{\mb{\ls}}

\newcommand{\iid}{\textrm{i.i.d.}\xspace}

\newcommand{\infdim}{d_{\cA}}

\newcommand{\gfilt}{\mathfrak{F}}
\newcommand{\hfilt}{\mathfrak{G}}

\newcommand{\val}{\mathsf{Val}}

\title{Beyond UCB:\\ Optimal and Efficient Contextual Bandits with Regression Oracles}
\author{
\begin{tabular}{c}
{\Large Dylan J. Foster~~~~~~~~~~Alexander Rakhlin}\\
Massachusetts Institute of Technology\\
{\small\texttt{\{dylanf,rakhlin\}@mit.edu}}
\end{tabular}
}
\date{}

\begin{document}
\maketitle

\begin{abstract}
A fundamental challenge in contextual bandits is to develop flexible, general-purpose algorithms with
computational requirements no worse than classical supervised learning
tasks such as classification and regression. Algorithms based on
regression have shown promising empirical success, but theoretical
guarantees have remained elusive except in special cases. We provide
the first universal and optimal reduction from contextual bandits to
\emph{online} regression. We show how to transform any oracle for online regression with a
given value function class into an algorithm for contextual
bandits with the induced policy class, with no overhead in runtime or memory
requirements. We characterize the minimax rates for contextual bandits
with general, potentially nonparametric function classes, and show that our algorithm is minimax optimal whenever the oracle obtains the optimal rate for regression. Compared to previous results, our algorithm requires no distributional assumptions beyond realizability, and works even when contexts are chosen adversarially.
\end{abstract}
\section{Introduction}
\label{sec:intro}
We consider the design of practical, provably efficient algorithms for
contextual bandits, where a learner repeatedly
receives contexts and makes decisions on the fly so as to learn a
policy that maximizes their total reward. Contextual bandits have
been successfully applied in user recommendation systems
\citep{agarwal2016making} and mobile health applications
\citep{tewari2017ads}, and in theory they are perhaps simplest reinforcement
learning problem that embeds the full complexity of statistical learning with function approximation.

A key challenge in contextual bandits is to develop flexible, general
purpose algorithms that work for arbitrary, user-specified classes of
policies and come with strong theoretical guarantees on
performance. Depending on the task, a user might wish to try decision trees,
kernels, neural nets, and beyond to get the best
performance. General-purpose contextual bandit algorithms ensure that
the user doesn't have to design a new algorithm from scratch every
time they encounter a new task.

Oracle-based algorithms constitute the dominant approach to
general-purpose contextual bandits. Broadly, these algorithms seek to
reduce the contextual bandit problem to basic supervised learning
tasks such as classification and regression so that off-the-shelf
algorithms can be applied. However, essentially all oracle-based contextual bandit algorithms suffer
from one or more of the following issues:
\begin{enumerate}[itemsep=3pt]
\item Difficult-to-implement oracle.\label{issue:1}
\item Strong assumptions on hypothesis class or distribution.\label{issue:2}
\item High memory and runtime requirements.\label{issue:3}
\end{enumerate}
Agnostic oracle-efficient algorithms
\citep{langford2008epoch,dudik2011efficient,agarwal2014taming} require
few assumptions on the distribution, but reduce contextual
bandits to \emph{cost-sensitive classification}. Cost-sensitive classification is
intractable even for simple hypothesis classes \citep{klivans2009cryptographic}, and in
practice implementations are forced to resort to heuristics to
implement the oracle
\citep{agarwal2014taming,krishnamurthy2016contextual}.

\citet{foster2018practical} recently showed that a variant of the
UCB algorithm for general function classes \citep{russo2014learning} can
be made efficient in terms of calls to an oracle for \emph{supervised
regression}. Regression alleviates some of the practical issues with
classification because it can be solved in closed form for simple
classes and is amenable to gradient-based methods. Indeed,
\citet{foster2018practical} and \citet{bietti2018contextual} found that this algorithm typically outperformed
algorithms based on classification oracles across a range of
datasets. However, the theoretical analysis of the algorithm relies on strong
distributional assumptions that are
difficult to verify in practice, and it can indeed fail pathologically
when these assumptions fail to hold.

All of the provably optimal general-purpose
algorithms described above---both classification- and regression-based---are memory hungry: they keep the entire
dataset in memory and repeatedly augment it before feeding it into the
oracle. Even if the oracle itself is online in the sense that it admits streaming or incremental updates, the
resulting algorithms do not have this property. At this point it
suffices to say that---to our knowledge---no general-purpose algorithm with provably
optimal regret has made it into a large-scale
contextual bandit deployment in the real world (e.g., \citet{agarwal2016making}).

In this paper, we address issues (1), (2), and (3) simultaneously: We
give a new contextual bandit algorithm which is efficient in terms of queries to
an \emph{online} oracle for \emph{regression}, and which requires
\emph{no assumptions} on the data-generating process beyond
a standard realizability assumption.

\subsection{Setup}
We consider the following contextual bandit protocol, which occurs 
over $T$ rounds. At each round $t\in\brk*{T}$, Nature selects a context $x_t\in\cX$
and loss function $\ls_t:\cA\to\brk*{0,1}$, where $\cA=\brk*{K}$ is
the learner's action space. The learner then selects an action
$a_t\in\cA$ and observes $\ls_t(a_t)$. We allow the contexts $x_t$ to be
chosen arbitrarily by an adaptive adversary, but we assume that each loss $\ls_t$ is
drawn independently from a fixed distribution
$\bbP_{\ls_t}(\cdot\mid{}x_t)$, where $\bbP_{\ls_1},\ldots,\bbP_{\ls_T}$ are selected a-priori by an oblivious adversary.

We assume that the learner has access to a class of value functions
$\cF\subset(\cX\times\cA\to\brk*{0,1})$ (such as linear models or
neural networks) that models the mean of the reward
distribution. Specifically, we make the following standard
\emph{realizability} assumption \citep{chu2011contextual,agarwal2012contextual,foster2018practical}.
\begin{assumption}[Realizability]
  \label{ass:realizability}
There exists a regressor $\fstar\in\cF$ such that for all $t$,
$\fstar(x,a) = \En\brk*{\ls_t(a)\mid{}x_t=x}$.
\end{assumption}
The learner's goal is to compete with the class of policies induced
by the model class $\cF$. For each regression function $f\in\cF$, we
let $\pi_f(x)=\argmin_{a\in\cA}f(x,a)$ be the induced policy. Then aim
of the learner is to minimize their \emph{regret} to the optimal policy:
\begin{equation}
  \label{eq:regret}
\Reg=  
\sum_{t=1}^{T}\ls_t(a_t)
-\sum_{t=1}^{T}\ls_t(\pistar(x_t)),
\end{equation}
where $\pistar\ldef{}\pi_{\fstar}$. Going forward, we let $\Pi=\crl*{\pi_f\mid{}f\in\cF}$ denote the induced
policy class.
\subsection{Contributions}
We introduce the notion of an \emph{online regression oracle}. At each
time $t$, an online regression oracle, which we denote \alg (for
``square loss regression algorithm''), takes as input
a tuple $(x_t,a_t)$, produces a real-valued prediction
$\yh_t\in\bbR$, and then receives the true outcome $y_t$. The goal of
the oracle is to predict the outcomes as well as the best function in
a class $\cF$, in the sense that for every sequence of outcomes
the \emph{square loss regret} is bounded:
\begin{equation}
  \label{eq:square_regret_informal}
  \sum_{t=1}^{T}(\yh_t-y_t)^{2} -
  \inf_{f\in\cF}\sum_{t=1}^{T}(f(x_t,a_t)-y_t)^{2} \leq{} \RegSquare.
\end{equation}
Our main algorithm, \mainalg (\pref{alg:main}), is a reduction that
efficiently and optimally turns any online regression oracle into an
algorithm for contextual bandits in the realizable setting.
\newtheorem*{thm:ub_informal}{Theorem \ref*{thm:reduction_main} (informal)}
\begin{thm:ub_informal}
  Suppose \pref{ass:realizability} holds. Then \mainalg, when invoked with an online regression oracle with square
  loss regret $\RegSquare$, ensures that with high probability
  \[
    \Reg \leq{} C\cdot\sqrt{KT\cdot\RegSquare},
  \]
  where $C>0$ is a small numerical constant. Moreover, \mainalg
  inherits the memory and runtime requirements of the oracle.
\end{thm:ub_informal}
We show (\pref{sec:minimax}) that \mainalg is \emph{optimal}, in the
sense that for every class $\cF$, there exists a choice for the oracle
\alg such that \mainalg attains the minimax optimal rate for
$\cF$. For example, when $\abs*{\cF}<\infty$, one can choose \alg such
that $\RegSquare\leq{}2\log\abs*{\cF}$, and so \mainalg enjoys the
optimal rate $\Reg\leq{}C\sqrt{KT\log\abs*{\cF}}$ for finite classes
\citep{agarwal2012contextual}. On the other hand, the reduction is
black-box in nature, so on the practical side one can simply choose
\alg to be whatever works best.

An advantage of working with 1) regression and 2) online oracles is
that we can instantiate \mainalg reduction to give new provable
end-to-end regret
guarantees for concrete function classes of interest. In
\pref{sec:algorithm} we flesh this direction out and provide new
guarantees for high-dimensional linear classes, generalized linear
models, and kernels. \mainalg is also robust to model misspecification: we show
(\pref{sec:misspecified}) that the performance gracefully degrades when the
realizability assumption is satisfied only approximately.

Compared to previous methods, which either maintain global confidence
intervals, version spaces, or distributions over feasible hypotheses,
our method applies a simple mapping proposed by \citet{abe1999associative} from scores to action
probabilities at each step. This leads to the method's efficient runtime guarantee. In \pref{sec:infinite}  we show that this type of reduction extend beyond the
finite actions by designing a variant of \mainalg that has $\Reg\leq{}C\sqrt{\infdim{}T\cdot\RegSquare}$ for the
setting where actions live in the $\infdim$-dimensional unit ball in $\ls_2$.

\subsection{Towards learning-theoretic guarantees for contextual bandits}

The broader goal of this work is to develop a deeper understanding of
the algorithmic principles and statistical complexity of contextual
bandit learning in the ``large-$\cF$, small-$\cA$'' regime, where the
goal is to learn from a rich, potentially nonparametric function class
with a small number of actions. We call this setting
``\textbf{C}ontextual \textbf{B}andits with \textbf{Ri}ch
\textbf{C}lasses of \textbf{H}ypotheses'', or \richcbs.

Beyond providing a general algorithmic principle for \richcbs
(\mainalg), we resolve two central questions regarding the statistical
complexity of \richcbs.
\begin{enumerate}
\item What are the minimax rates for \richcbs when
  $\abs*{\cF}=\infty$?
\item Can we achieve logarithmic regret for \richcbs when the
  underlying instance has a gap?
\end{enumerate}
Recall that for general finite classes $\cF$, the gold standard here is
$\Reg\leq\sqrt{KT\log\abs*{\cF}}$, with an emphasis on the
logarithmic scaling in $\abs*{\cF}$. For the first point, we characterize (\pref{sec:minimax}) the minimax rates for infinite
classes $\cF$ as a function of \emph{metric
  entropy}, a fundamental complexity measure in learning theory.
We also show that \mainalg is universal, in the sense that
it can always be instantiated with a choice of \alg to achieve the
minimax rate. Interestingly, we show that for general classes with
metric entropy $\cH(\cF,\veps)$, the
minimax rate is $\wt{\Theta}(T\cdot\veps_{T})$, where $\veps_{T}$ satisfies the
classical balance
\[
  \veps_{T}^{2} \asymp \cH(\cF,\veps_T)/T,
\]
found throughout the literature on nonparametric estimation \citep{yang1999information,tsybakov2008introduction}.

For the second point, we show (\pref{sec:gap}), that for general
function classes $\cF$ with $\abs*{\cF}<\infty$, obtaining logarithmic regret when there is a gap between the best and second-best action is impossible if we
insist that regret scales with $\polylog\abs*{\cF}$: There
exist instances with constant gap and polynomially large hypothesis
class for which any algorithm must experience $\sqrt{T}$-regret.

This last point suggests that designing optimal algorithms for \richcbs seems to require new
algorithmic ideas. Indeed, two of the dominant strategies for the
realizable setting, generalized UCB and Thompson sampling \citep{russo2013eluder}, always adapt to the
gap to get logarithmic regret, but without strong structural
assumptions on $\cF$ they can have regret $\Omega(\abs*{\cF})$.

\subsection{Related work}
Our algorithm builds off of the work of \citet{abe1999associative} (see
also \citet{abe2003reinforcement}). Our key insight is that a
particular action selection scheme used in these works for linear
contextual bandits actually yields an algorithm for general function
classes when combined with the idea of an online regression oracle. Interestingly, while \citet{abe1999associative} contains
essentially the first formulation of the contextual bandit problem,
the techniques used within seem to have been forgotten by time
in favor of more recent approaches to linear contextual bandits
\citep{abbasi2011improved,chu2011contextual}; see further discussion in \pref{sec:algorithm}.

As discussed in the introduction, our results build on a long line of
work on oracle-efficient contextual bandit algorithms. We discuss some important points of comparison below.

\paragraph{Agnostic algorithms.}
The longest line of research on oracle-efficient CBs focuses on the
agnostic \iid setting
\citep{langford2008epoch,dudik2011efficient,agarwal2014taming}. All of
these algorithms assume access to an \emph{offline} cost-sensitive
classification oracle for the policy class which, given a dataset
$(x_1,\ls_1),\ldots,(x_n,\ls_n)$, solves
\begin{equation}
  \argmin_{\pi\in\Pi}\sum_{t=1}^{n}\ls_t(\pi(x_t)).\label{eq:csc_oracle}
\end{equation}
In particular, the \iltcblong (\iltcb) algorithm
\citep{agarwal2014taming} enjoys optimal $\sqrt{KT\log\abs*{\Pi}}$
regret given such an oracle. This type of oracle has two
drawbacks. First, classification for arbitrary datasets is intractable
for most policy classes, so implementations typically resort to
heuristics to implement \pref{eq:csc_oracle}. Second, because the
oracle is \emph{offline}, the memory required by \iltcb scales
linearly with $T$ (the algorithm repeatedly generates augmented
versions of the dataset and feeds them into the oracle). To deal with
this issue the implementation of \iltcb in \cite{agarwal2014taming}
resorts to heuristics in order to make use of an online oracle
classification, but the resulting algorithm has no guarantees, and
analyzing it was left as an open problem.

A parallel line of work focuses on algorithms for the
\emph{adversarial} setting where losses are also arbitrary
\citep{rakhlin2016bistro,syrgkanis2016efficient,syrgkanis2016improved}. Notably,
the \bistro algorithm \citep{rakhlin2016bistro} essentially gives a
reduction from adversarial CBs to a particular class of
``relaxation-based'' online learning algorithms for cost-sensitive classification,
but the algorithm has sub-optimal $T^{3/4}$ regret for finite classes.
\paragraph{Realizability-based algorithms.}
Under the realizability assumption, \citet{foster2018practical}
provide a version of the UCB strategy for general
function classes \citep{russo2014learning} that makes use of a
\emph{offline regression oracle} that solves
\begin{equation}
  \label{eq:reg_oracle}
\argmin_{f\in\cF}\sum_{t=1}^{n}(f(x_t,a_t)-\ls_t(a_t))^{2}.  
\end{equation}
While this is typically an easier optimization problem than
\pref{eq:csc_oracle}---it can be solved in closed form for linear
classes and is amenable to gradient-based methods---the algorithm only
attains optimal regret under strong distributional assumptions (beyond just
realizability) or when the class $\cF$ has bounded eluder dimension \citep{russo2013eluder}, and it can have linear regret when these assumptions
fail to hold \citep[Proposition 1]{foster2018practical}.

Thompson sampling and posterior sampling are closely related to UCB
and have similar regret guarantees \citep{russo2014learning}. These
algorithms are only efficient for certain simple classes $\cF$, and
implementations for general classes resort to heuristics such as
bootstrapping, which do not have strong theoretical guarantees except
in special cases \citep{vaswani2018new,kveton2019garbage}.

We mention in passing that under our assumptions
(realizability, online regression oracle), one can design an
online oracle-efficient variant of \egreedy with $T^{2/3}$-type
regret; \mainalg appears to be strictly superior.

\paragraph{Other square loss-related reductions.}~\citet{abernethy2013large} consider the related problem reducing realizable contextual bandits
with general function classes $\cF$ \emph{and} large action spaces to
knows-what-it-knows (KWIK) learning oracles \citep{li2011knows}. KWIK learning is much stronger property than regret
minimization, and KWIK learners only exist for certain structured
hypotheses classes. Interestingly though, this work also provides a
computational lower bound which suggests that efficient reductions of
the type we provide here (\mainalg) are \emph{not} possible if one
insists on $\log{}K$ dependence rather than $\poly(K)$ dependence.

\citet{abbasi2012online} develops contextual bandit algorithms that use
online regression algorithms to form confidence sets for use within
UCB-style algorithms. Ultimately these algorithms inherit the usual
drawbacks of UCB, namely that they require either strong assumptions on
the structure of $\cF$ or strong distributional assumptions.

\subsection{Additional notation}
	We adopt non-asymptotic big-oh notation: For functions
	$f,g:\cX\to\bbR_{+}$, we write $f=\bigoh(g)$ if there exists some constant
	$C>0$ such that $f(x)\leq{}Cg(x)$ for all $x\in\cX$. We write $f=\bigoht(g)$ if $f=\bigoh(g\max\crl*{1,\mathrm{polylog}(g)})$.

	For a vector $x\in\bbR^{d}$, we let $\nrm*{x}_{2}$ denote the euclidean
	norm and $\nrm*{x}_{\infty}$ denote the element-wise $\ls_{\infty}$
	norm. For a matrix $A$, we let $\nrm*{A}_{\op}$ denote the
	operator norm. If $A$ is symmetric, we let $\eigmin(A)$ denote the
	minimum eigenvalue. When $P\psdgt{}0$ is a positive definite matrix,
	we let $\nrm*{x}_{P}=\sqrt{\tri*{x,Px}}$ denote the induced weighted
	euclidean norm. 


\section{The reduction: \mainalg}
\label{sec:algorithm}

We now describe our main algorithm, \mainalg, and state our main regret
guarantee and some consequences for concrete function
classes. To give the guarantees, we first formalize the concept
of an online regression oracle, as sketched in the introduction.
\subsection{Online regression oracles}
We assume access to an oracle $\alg$ for the standard online
learning setting with the square
loss \citep[Chapter 3]{PLG}. The oracle performs real-valued online regression with
features in $\cZ\ldef\cX\times\cA$, and is assumed to have a
prediction error guarantee relative to the regression function class
$\cF$. We consider the following model:
\begin{itemize}
\item[] For $t=1,\ldots,T$:
  \begin{itemize}
  \item Nature chooses input instance $z_t=(x_t,a_t)$.
  \item Algorithm chooses prediction $\yh_t$.
  \item Nature chooses outcome $y_t$.
  \end{itemize}
\end{itemize}
Formally, we model the algorithm as a sequence of mappings
$\alg_t:\cZ\times\prn*{\cZ\times{}\bbR}^{t-1}\to\brk*{0,1}$, so that
$\yh_t=\alg_t(z_t\midsem(z_1,y_1),\ldots,(z_{t-1},y_{t-1}))$ in the protocol above. Each such algorithm induces a mapping
\begin{equation}
\yh_t(x,a)
\ldef{}\alg_t(x,a\midsem(z_1,y_1),\ldots,(z_{t-1},y_{t-1})),\label{eq:yhat}
\end{equation}
which corresponds to the prediction the algorithm would make at time
$t$ if we froze its internal state and fed in the feature vector $(x,a)$.

The simplest condition under which our
reduction works posits that \alg enjoys a regret bound for individual
sequence prediction.
\begin{customassumption}{2a}
  \label{ass:square_regret}
The algorithm \alg{} guarantees that for every (possibly adaptively
chosen) sequence $z_{1:T},y_{1:T}$, regret is bounded as
\begin{equation}
  \sum_{t=1}^{T}\prn*{\yh_t - y_t}^{2} -
    \inf_{f\in\cF}\sum_{t=1}^{T}\prn*{f(z_t) - y_t}^{2} \leq{}
    \RegSquare.\label{eq:square_regret}
  \end{equation}
\end{customassumption}
While there is a relatively complete theory characterizing what regret
bounds $\RegSquare$ can be achieved for this setting for general
classes $\cF$ \citep{rakhlin2014nonparametric}, the requirement that the regret bound holds for
arbitrary sequences $y_{1:T}$ may be restrictive for some classes, at least as far
as efficient algorithms are concerned. The following relaxed
assumption also suffices.
\begin{customassumption}{2b}
  \label{ass:square_regret2}
Under \pref{ass:realizability}, the algorithm \alg{} guarantees that for every (possibly adaptively
chosen) sequence $\crl*{(x_t,a_t)}_{t=1}^{T}$, we have
\begin{equation}
  \sum_{t=1}^{T}\prn*{\yh_t - \fstar(x_t,a_t)}^{2} \leq{}
  \RegSquare.\label{eq:square_regret2}
\end{equation}
\end{customassumption}
\pref{ass:square_regret2} holds with high probability whenever
\pref{ass:square_regret} holds and the problem is realizable, but it
is a weaker condition that allows for algorithms tailored toward realizability; we shall see examples of this in the
sequel. This formulation shows that the choice of square loss in \pref{eq:square_regret}
does not actually play a critical role: Any algorithm that attains a regret bound of the form
\pref{eq:square_regret} with the square loss replaced by a
\emph{strongly convex} loss such as the log loss implies a bound of the
type \pref{eq:square_regret2} under realizability.

\subsection{The algorithm}

Our main algorithm, \mainalg, is presented in \pref{alg:main}. At time
$t$, the algorithm receives the context $x_t$ and computes the
oracle's predicted scores $\yh_t(x_t,a)$ for each action. Then, following the
probability selection scheme of \citet{abe1999associative}, it computes
the action with the lowest score ($b_t$) and assigns a probability to
every other action inversely proportional to the gap
between the action's score and that of $b_t$. Finally, the algorithm samples
its action $a_t$ from this distribution, observes the loss
$\ls_t(a_t)$, and feeds the tuple $((x_t,a_t),\ls_t(a_t))$ into the
oracle. The main guarantee for the algorithm is as follows.
\begin{theorem}
  \label{thm:reduction_main}
  Suppose \pref{ass:realizability} and
  \savehyperref{ass:square_regret}{Assumption 2a/b} hold.
  Then for any $\delta>0$, by setting $\mu=K$ and $\gamma =
  \sqrt{KT/\prn{\RegSquare+\log(2\delta^{-1})}}$, \mainalg guarantees that with probability at least $1-\delta$,
  \begin{equation}
    \label{eq:reduction_regret}
    \Reg \leq{} 4\sqrt{KT\cdot\RegSquare} + 8\sqrt{KT\log(2\delta^{-1})}.
  \end{equation}
\end{theorem}
\noindent{}Let us discuss some key features of the algorithm and its regret bound.
\begin{itemize}
\item The algorithm enjoys $\bigoht(\sqrt{T})$-regret whenever the oracle \alg gets a
  fast $\log{}T$-type rate for online regression. This holds for
  finite classes ($\RegSquare=\log\abs*{\cF}$) as well as parametric
  classes such as linear functions in $\bbR^{d}$
  ($\RegSquare=d\log{}(T/d)$). We sketch some more examples below, and
  we show in \pref{sec:minimax} that the regret is optimal whenever
  $\alg$ is optimal.
\item The algorithm inherits the runtime and memory requirements of
  the oracle \alg up to lower order terms. If $\cT_{\alg}$ denotes per-round runtime for \alg and $\cM_{\alg}$ denotes the maximum memory,
  then the per-round runtime of \mainalg is $\bigoh(\cT_{\alg}\cdot{}K)$, and the
  maximum memory is $\bigoh(\cM_{\alg}\cdot{}K)$.
\item The regret scales as $\sqrt{K}$ in the number of actions. This
  is near-optimal in the sense that any algorithm that works \emph{uniformly} for all
  oracles must pay a $\bigomt(\sqrt{K})$ factor: For multi-armed
  bandits, one can achieve $\RegSquare=\log{}K$,\footnote{This can be
    achieved through Vovk's aggregating algorithm
    \citep{vovk1995game}.} yet the optimal bandit regret is
  $\Omega(\sqrt{KT})$. However, for specific function classes, the dependence on $K$ may be
  suboptimal (e.g., for linear classes, regret can be made to scale only with $\log{}K$
  \citep{chu2011contextual}).
\end{itemize}
\begin{algorithm}[t]
  \setstretch{1.1}
  \begin{algorithmic}[1]
    \State \textbf{parameters}:
    \Statex{}~~~~Learning rate $\gamma>0$, exploration parameter $\mu>0$.
    \Statex{}~~~~Online regression oracle $\alg$.
    \For{$t=1,\ldots,T$}
    \State Receive context $x_t$.
    \Statex{}~~~~~\algcomment{Compute oracle's predictions (Eq.\pref{eq:yhat}).}
    \State For each action $a\in\cA$, compute
    $\yh_{t,a} \ldef \yh_t(x_t,a)$.
    \State Let $b_t=\argmin_{a\in\cA}\yh_{t,a}$.
    \State For each $a\neq{}b_t$, define $p_{t,a} = \frac{1}{\mu +
      \gamma(\yh_{t,a} -\yh_{t,b_t})}$, and let $p_{t,b_t} = 1-\sum_{a\neq{}b_t}p_{t,a}$.\label{line:prob}
    \State  Sample $a_{t}\sim{}p_t$ and observe loss $\ls_t(a_t)$.
    \State Update $\alg$ with example $((x_t,a_t),\ls_t(a_t))$.
    \EndFor
  \end{algorithmic}
  \caption{\mainalg}
  \label{alg:main}
\end{algorithm} 

At a conceptual level, the proof (which, beyond the idea of using a
generic regression oracle and taking advantage of modern martingale
tail bounds, closely follows
\citet{abe1999associative}) is interesting because it is agnostic to the structure of the class $\cF$. We show that at each
timestep, the instantaneous bandit regret is upper bounded by the
instantaneous square loss regret of \alg. No structure is shared
across timesteps, and all of the heavy lifting regarding
generalization is taken care of by \pref{ass:square_regret}/\pref{ass:square_regret2}.

One
important point to discuss is the assumption that the bound
\pref{eq:square_regret2} holds for every sequence
$\crl*{(x_t,a_t)}_{t=1}^{T}$. While the assumption that the bound
holds for adaptively chosen contexts $x$ can be removed if contexts are \iid,
the analysis critically uses that the regret bound holds when the
actions $a_1,\ldots,a_{T}$ are chosen adaptively (since actions
selected in early rounds are used by \mainalg to determine the action
distribution at later rounds). On a related note, even when contexts are \iid, it is not
clear that one can implement an online regression oracle that
satisfies the requirements of \pref{thm:reduction_main} via calls to
an offline regression oracle, and offline versus online
regression oracles appear to be incomparable assumptions.  Whether optimal regret can be attained via reduction to an offline oracle remains an
interesting open question. 

\subsection{Examples and applications}
\label{sec:examples}
Online square loss regression is a well-studied problem, and
efficient algorithms with provable regret guarantees are known for
many classes
\citep{Vovk98,azoury2001relative,Vovk06metric,gerchinovitz2013sparsity,rakhlin2014nonparametric,gaillard2015chaining}.
Here we take advantage of these results by instantiating \alg within
\mainalg to derive end-to-end regret guarantees for various
classes---some new, some old.
\paragraph{Low-dimensional linear classes.}
We first consider
the familiar \linucb setting,
where \begin{equation}\cF=\crl*{(x,a)\mapsto\tri*{\theta,x_a}\mid{}\theta\in\bbR^{d},\nrm*{\theta}_2\leq{}1},\label{eq:linear_class}\end{equation}
and $x=\prn*{x_a}_{a\in\cA}$, where $x_a\in\bbR^{d}$ has $\nrm*{x_a}_2\leq{}1$. Here \linucb obtains $\Reg\leq{}\bigoh(\sqrt{dT\log^{3}(KT)})$ \citep{chu2011contextual}. By choosing \alg to be the
Vovk-Azoury-Warmuth forecaster, which has $\RegSquare\leq{}d\log(T/d)$
\citep{Vovk98,azoury2001relative}, \mainalg has
$\Reg\leq{}\bigoh(\sqrt{dKT\log(T/d)})$.\footnote{In order satisfy the condition that predictions $\yh_t$ are bounded, we must use a variant of Vovk-Azoury-Warmuth with projection onto the $\ls_2$ ball. This can easily be achieved using, e.g., the analysis in \cite{orabona2015generalized}.} While this has worse dependence on
$K$ (square root rather than logarithmic), the resulting algorithm works
when contexts are chosen by an adaptive adversary, whereas \linucb
requires an oblivious adversary. It would be interesting to understand
whether such a tradeoff is optimal. We also remark that---ignoring dependence on $K$---the algorithm
precisely matches a recently established lower bound of
$\Omega(\sqrt{dT\log(T/d)})$ for this setting \citep{li2019nearly}.
\paragraph{High-dimensional linear classes and Banach spaces.}
In the same setting as above, by choosing \alg to be Online Gradient
Descent, we obtain $\RegSquare\leq{}\bigoh(\sqrt{T})$, and consequently
$\Reg\leq{}\bigoh(K^{1/2}\cdot{}T^{3/4})$. This rate is interesting because
it has worse dependence on the time-horizon $T$, but is completely
\emph{dimension-independent}, and the algorithm runs in linear time,
which is considerably faster than \linucb ($\bigoh(d^{2})$ per step). This
result generalizes the \textsf{BW} algorithm of \citet{abe2003reinforcement}, who gave the same
bound for the setting where rewards are binary, and showed that $T^{3/4}$ is optimal
when $d$ is large. We believe this tradeoff between dimension
dependence and $T$ dependence has been somewhat overlooked and merits
further investigation, especially as it pertains to practical algorithms.

For a more general version of this result, we let
$(\Bspace,\nrm*{\cdot})$ be a separable Banach space
and take
\[\cF=\crl*{(x,a)\mapsto\tri*{\theta,x_a}\mid{}\theta\in\Bspace,\nrm*{\theta}\leq{}1},\]
where $x_a$ to belongs to the dual space
$(\Bspace^{\star},\nrm*{\cdot}_{\star})$ and has $\nrm*{x_a}_{\star}\leq{}1$. For
this setting, whenever $\Bspace$ is $(2,D)$-uniformly convex, Online
Mirror Descent can be configured to have $\RegSquare\leq{}\sqrt{T/D}$
\citep{srebro2011universality}, and \mainalg consequently has
$\Reg\leq{}\bigoh(K^{1/2}\cdot{}T^{3/4}D^{-1/4})$. This leads to linear
time algorithms with nearly dimension-free rates for, e.g., $\ls_1$-
and nuclear norm-constrained linear classes.
\paragraph{Kernels.}
Suppose that $\cF$ is a reproducing kernel Hilbert space with RKHS norm
$\nrm*{\cdot}_{\cH}$ and kernel $\cK$. Let $\nrm*{f}_{\cH}\leq{}1$ for
all $f\in\cH$, and assume $\cK(x_a,x_a)\leq{}1$ for all $x\in\cX$. A
simple observation is that, since Online Gradient Descent kernelizes,
the $\bigoh(T^{3/4})$ regret bound from the previous example immediately
extends to this setting. This appear to be a new result; Previous work
on kernel-based contextual bandits \citep{valko2013finite} gives
regret bounds of the form $\sqrt{d_{\mathrm{eff}}T}$, assuming that
the effective dimension $d_{\mathrm{eff}}$  of the empirical design
matrix is bounded. Again there is a tradeoff, since our result requires no
assumptions on the data beyond bounded RKHS norm, but has worse
(albeit optimal under these assumptions) dependence on the time
horizon.
\paragraph{Generalized linear models.}
Let $\sigma:\bbR\to\brk*{0,1}$ be a fixed non-decreasing $1$-Lipschitz link function,
and
let \[\cF=\crl*{(x,a)\mapsto\sigma\prn*{\tri*{\theta,x_a}}\mid{}\theta\in\bbR^{d},\nrm*{\theta}_2\leq{}1},\]
where we again take $\nrm*{x_a}_{2}\leq{}1$. For this setting, under
the realizability assumption, the \glmtron algorithm
\citep{kakade2011efficient} satisfies
\pref{ass:square_regret2}, in the sense that is has
\[
\sum_{t=1}^{T}\prn*{\yh_t-\sigma(\tri*{\theta^{\star},x_{a_t}})}^{2}
\leq{} \bigoh(\sqrt{T}),
\]
where $\fstar(x,a) = \sigma(\tri*{\theta^{\star},x_a})$; see
\pref{prop:glm1} in \pref{app:details} for details. This leads to a
dimension-free regret bound $\Reg\leq{}\bigoh(T^{3/4})$, similar
to the linear setting. If we have a lower bound on the link function derivative (i.e.,
$\sigma'\geq{}c_{\sigma}>0$), then a second-order variant of \glmtron
(\pref{prop:glm2}) satisfies \pref{ass:square_regret2} with
$\RegSquare = \bigoh(d\log{}T/c_{\sigma}^2)$. Plugging this into
\mainalg gives regret $\bigoh(\sqrt{dKT\log{}T/c_{\sigma}^2})$. This
matches the dependence on $d$ and $T$ in previous results for generalized linear contextual bandits with finite actions
\citep{li2017provably}, but unlike these results the algorithm does not require stochastic
contexts, and requires no assumptions on the design matrix
$\frac{1}{T}\sum_{t=1}^{T}x_{t,a_t}x_{t,a_t}^{\trn}$ or its population
analogue.

\subsection{Minimax perspective}
The analysis of \mainalg is interesting because
the reduction from square loss regret to contextual bandit regret
completely ignores the structure of the function class $\cF$. At a
high level, the proof proceeds by showing that the probability
selection strategy ensures that
\begin{equation}
  \sum_{t=1}^{T}\En_{a\sim{}p_t}\brk*{\fstar(x_t,a)-\fstar(x_t,\pistar(x_t))}
  \leq{} \frac{2KT}{\gamma} +
  \frac{\gamma}{4}\sum_{t=1}^{T}\En_{a\sim{}p_t}\brk*{\prn*{\yh_{t,a}-\fstar(x_t,a)}^{2}}.\label{eq:reg_to_square}
\end{equation}
at which point we can bound the right-hand side by using the regret
bound for \alg. In fact, the probability selection strategy in \mainalg actually gives a
stronger guarantee than \pref{eq:reg_to_square}. Consider the
following \emph{per-round minimax} problem:
\begin{equation}
  \label{eq:minimax}
\val(\gamma) \ldef{}
\max_{\yh\in\brk*{0,1}^{K}}\min_{p\in\Delta_K}\max_{\fstar\in\brk*{0,1}^{K}}\max_{\astar\in\brk*{K}}\En_{a\sim{}p}\brk*{
  \fstar_{a} - \fstar_{\astar} - \frac{\gamma}{4}\prn*{\yh_a - \fstar_a}^{2}
  }.
\end{equation}
If $\val(\gamma)\leq{}c$, we can interpret this as saying, ``For every choice of $\yh$, there exists an
action distribution such that regardless of the value of $f^{\star}$,
the immediate regret with respect to $\fstar$ is bounded by the squared
prediction error of $\yh$, plus a constant $c$.'' The probability selection rule used in
\mainalg with parameter $\gamma$ certifies that
$\val(\gamma)\leq{}\frac{2K}{\gamma}$. The takeaway is that, at the
level of the reduction, $\fstar(x_t,a)$ might as well be chosen
adversarially at each round rather than realized by a specific
function $\fstar\in\cF$ chosen a-priori. We are hopeful that this per-round
minimax approach to reductions will be more broadly useful, and
indeed our extension to infinite actions in \pref{sec:infinite} uses
similar per-round reasoning. To close the section, we give a lower bound on the minimax value which
shows that the action selection strategy used in \mainalg is
near-optimal for the minimax problem \pref{eq:minimax}.
\begin{proposition}
  \label{prop:minimax_lb}
  For any $\gamma \geq{}2$, we have
  $\val(\gamma) \geq{} \frac{(1-1/K)}{\gamma}$.
\end{proposition}


\section{Optimality and universality}
\label{sec:minimax}
In light of \pref{thm:reduction_main}, a natural question is whether one can always instantiate \mainalg such that its regret is optimal for the class $\cF$ under consideration. More broadly, we seek to understand the minimax rate for the \richcb setting where $\cF$ is a large, potentially nonparametric function class and the problem is realizable. In this section we first prove a lower bound on minimax regret achievable for any function class $\cF$. We then show that \mainalg is \emph{universal}, in the sense that there always exist a choice for $\alg$ that achieves the lower bound (up to dependence on the number of actions, which is not our focus).

For technical reasons, we make two simplifying assumptions in this
section. First, we focus on the setting where $(x_t,\ls_t)$ are drawn
\iid from a joint distribution $\mu$. Second, we assume that the regression function class $\cF$ \emph{tensorizes}:
There is a base function class $\cG\subseteq(\cX\to\brk*{0,1})$ such
that $\cF=\cG^{K}$, in the sense $\cF$ consists of functions of the form
\[
f(x,a) = g_a(x),
\]
where $g_a\in\cG$.

Our upper and lower bounds are stated in terms of the \emph{metric entropy} of the base class $\cG$.
For a sample set $S=\{x_1,\ldots,x_n\}$, let $\cN_{2}(\cG,\veps,S)$ denote the size of the smallest set $\cG'$ such that
\[
  \displaystyle
\forall{}g\in\cG,\;\;\exists{}g'\in\cG'\;\;\text{s.t.}\;\;\prn*{\frac{1}{n}\sum_{t=1}^{n}(g(x_t)-g'(x_t))^{2}}^{1/2}\leq{}\veps.
\]
The \emph{empirical entropy} of $\cG$ is then defined as
\begin{equation}
\Hiid(\cG,\veps) = \sup_{n\geq{}1,S\in\cX^{n}}\log\cN_{2}(\cG,\veps,S).
\end{equation}
Empirical entropy is a fundamental quantity in statistical learning that is both necessary and sufficient for
learnability, as well as polynomially related to other standard complexity
measures such as (local) Rademacher complexity and fat-shattering
dimension \citep{StatNotes2012}. We give concrete examples in the sequel, but for now we make the following assumption.
\begin{customassumption}{3}
  \label{ass:iid_me}
  Contexts and losses are drawn i.i.d. from a joint distribution $\mu$, and there
    exists a constant $p>0$ such that for all $\veps>0$, the empirical entropy for $\cG$ grows as
  \[
    \Hiid(\cG,\veps) \lesssim \veps^{-p}.
  \]
\end{customassumption}
Our upper and lower bounds characterize the optimal regret for \richcbs as a function of the growth rate parameter $p>0$ in \pref{ass:iid_me}. We first state the lower bound.
\begin{theorem}[Lower bound]
  \label{thm:lb_stochastic}
  Let $\cG$ be any function class for which
  $\Hiid(\cG,\veps)=\Theta\prn*{\veps^{-p}}$ for some $p>0$. Then
  there exists a slightly modified class $\cG'$ with
  $\Hiid(\cG',\veps)=\wt{\Theta}\prn*{\veps^{-p}}$ for which the
  corresponding function class $\cF$ (with $K=2$) is such that any algorithm must have
  \begin{equation}
    \label{eq:lb}
    \En\brk*{\Reg}\geq{}\bigomt\prn*{
      T^{\frac{1+p}{2+p}}
    },
  \end{equation}
  on some realizable instance for $\cF$.
\end{theorem}
We now show that \mainalg can always be instantiated to match the lower bound \pref{eq:lb} in terms of dependence on $T$.
\begin{theorem}[Universality of \mainalg]
  \label{thm:universal_ub_tensor}
    Whenever \pref{ass:iid_me} holds, there exists a choice for
  the base regret minimization algorithm \alg such that with
  probability at least $1-\delta$, \mainalg{} has
  \[
    \Reg\leq{}\bigoht\prn*{(KT)^{\frac{1+p}{2+p}} + \sqrt{K^{2}T\log(\delta^{-1})}}.
  \]
\end{theorem}
The idea behind the proof of \pref{thm:universal_ub_tensor} is to
choose \alg to run Vovk's aggregating algorithm over an empirical cover for $\cG$. The main technical difficulty is that we must find a cover that is close on the distribution $\mu$, which the algorithm has no prior knowledge of. To get around this issue, the algorithm continually refines a cover based on data collected so far.
\paragraph{Examples.}
Let us make matters slightly more concrete and show how to extract some familiar regret bounds from \pref{thm:lb_stochastic} and \pref{thm:universal_ub_tensor}. First, for linear classes \pref{eq:linear_class} (specifically, the tensorized variants), one has $\Hiid(\cG,\veps) \propto d\log(1/\veps)\wedge{}\veps^{-2}$ \citep{zhang2002covering}, and hence the theorems recover the $\sqrt{dT}$ and $T^{3/4}$ regret bounds for linear classes described in the previous section.

\citet{slivkins2011contextual} derives fairly general results for
nonparametric contextual bandits. As one example, their results imply
that when $\cG$ is the set of all $1$-Lipcshitz functions over
$\brk*{0,1}^{d}$, the optimal regret is $T^{\frac{1+d}{2+d}}$. Since
such classes have $\Hiid(\cG,\veps)\propto\veps^{-d}$, our theorems
recover this result. Similarly, for H\"older-smooth functions of order
$\beta$, we have $p=d/\beta$ which yields the rate $T^{\frac{d+\beta}{d+2\beta}}$ \citep{rigollet2010nonparametric}.

As a final example, \citet{bartlett2017spectrally} show that neural
networks with appropriately bounded spectral norm and $\ls_{2,1}$ norm
have $\Hiid(\cG,\veps)\propto\veps^{-2}$. Our theorems imply that
$\wt{\Theta}(T^{3/4})$ is optimal for such models.
\paragraph{Discussion.}
The assumptions made in this section (tensorization, stochastic contexts) can be relaxed, but we do not have a complete picture of the optimal regret for all values of $p$ in this case. For adversarial contexts and without the tensorization assumption, if $\cF$ has bounded \emph{sequential} metric entropy then Theorem 1 of \citet{rakhlin2014nonparametric} implies that there exists a choice for \alg such that $\RegSquare\leq{}T^{1-\frac{2}{2+p}}$ and thus \mainalg has $\Reg\leq{}\bigoh(T^{\frac{1+p}{2+p}})$ as in \pref{thm:universal_ub_tensor}, but only for $p\leq{}2$. On the other hand, for stochastic contexts it is also possible to show that a variant of the algorithm in \pref{thm:universal_ub_tensor} based on slightly different concentration arguments matches the regret bound $\bigoh(T^{\frac{1+p}{2+p}})$ without the tensorization assumption, but only for $p\geq{}1$. Resolving the optimal dependence on $K$ seems challenging
and likely requires more refined complexity measures; see also
discussion in \citet{daniely2015multiclass}.

Previous works have given
regret bounds for infinite policy classes
that depend on the complexity (e.g., VC dimension) of the policy class
\citep{beygelzimer2011contextual,foster2018contextual}. These
guarantees are somewhat different than the ones we provide here, which
depend on the complexity of the regression function class $\cF$ rather
than the class of policies it induces (but require realizability).


\section{On gap-dependent regret bounds}
\label{sec:gap}

In this section we give some negative results regarding
instance-dependent regret bounds for \richcbs. Since \pref{thm:reduction_main} recovers the usual
$\bigoht(\sqrt{KT})$ bound for multi-armed bandits, a natural question
is whether the algorithm can recover \emph{instance-dependent} regret
bounds of the form $\bigoh(\frac{K\log{}T}{\Delta})$ when there is a gap $\Delta$ between the best and second-best action. More ambitiously,
can the algorithm achieve similar instance-dependent regret bounds for
rich function classes $\cF$?

To address this question, we assume Bayes regressor $\fstar$ enjoys
a gap between the optimal and second-best action for every
context. The following definition is adapted from \citet{dani2008stochastic}.
\begin{definition}[Uniform gap]
  A contextual bandit instance is said to have \emph{uniform gap}
  $\Delta$ if for all $x\in\cX$,
  \[
    \fstar(x,a) - \fstar(x,\pistar(x)) >
    \Delta\quad\forall{}a\neq{}\pistar(x).
  \]
\end{definition}
We would like to understand whether \pref{thm:reduction_main} can be
improved when the uniform gap condition holds. For example, is it possible to
select the learning rate $\gamma$ such that \mainalg has
\begin{equation}
\Reg \leq{}
\frac{K\RegSquare}{\Delta}\cdot\mathrm{polylog}(T)?\label{eq:gap_reduction}
\end{equation}
As a special case, such a regret bound would recover the $\bigoht(\frac{K}{\Delta})$-type regret bound for
multi-armed bandits by choosing $\alg$ with the exponential weights
strategy. More generally, for any finite class $\cF$, the hypothesized bound \pref{eq:gap_reduction} would imply a
regret bound of
\begin{equation}
\Reg \leq{} \bigoht\prn*{\frac{K\log\abs*{\cF}}{\Delta}}\label{eq:gap_finite}
\end{equation}
by taking $\alg$ to be Vovk's aggregating algorithm, which has
$\RegSquare=\log\abs*{\cF}$. Here we give an information-theoretic
lower bound which shows that such a
regret bound is not possible, not just for \mainalg but for \emph{any}
contextual bandit algorithm.
\begin{theorem}
  \label{thm:gap_lb}
  For every $T$, there exists a function class $\cF$ with two arms and
  $\abs*{\cF}\leq{}\sqrt{2T}$
  such that for any (potentially randomized) contextual bandit algorithm, there
  exists a realizable and noiseless contextual bandit instance with
  uniform gap $\Delta=\frac{1}{4}$ on which
  \[
    \En\brk*{\Reg} \geq{} \frac{1}{16}\sqrt{T}.
  \]
\end{theorem}
The function class in \pref{thm:gap_lb} has
$\abs*{\cF}=\bigoh(\sqrt{T})$, and all instances considered in the
theorem have constant gap. For such setups, the hypothesized regret
bound \pref{eq:gap_finite} would give $\Reg\leq{}\bigoht(1)$. Hence,
\pref{thm:gap_lb} rules out \pref{eq:gap_finite} and
\pref{eq:gap_reduction}, and in fact rules out any regret bound of the
form $\Reg \leq{}
\bigoht\prn*{\frac{K\log\abs*{\cF}}{\Delta}\cdot{}T^{1/2-\veps}}$ for
constant $\veps$. 

In essence, the theorem shows that to obtain
instance-dependent regret guarantees, one can at best hope for regret
bounds that scale with $\frac{\abs*{\cF}}{\Delta}$ rather than
$\frac{\log\abs*{\cF}}{\Delta}$. In other words, instance-dependent regret is at odds with learning from rich function classes (\richcbs), where regret scaling with $\abs*{\cF}$ is unacceptable.
It is known that for linear function classes $\cF$, and more broadly
function classes that satisfy certain structural assumptions such as
bounded \emph{eluder dimension} \citep{russo2013eluder}, gap-dependent
logarithmic regret bounds are achievable through variants of UCB and
Thompson sampling. However, bounded eluder
dimension is a rather strong assumption which is essentially only known to hold
for linear models, generalized linear models, and classes for which
the domain size $\abs*{\cX}$ is bounded.\footnote{There is no
  contradiction with \pref{thm:gap_lb}, as the construction in the
  theorem scales $\abs*{\cX}$ as $\sqrt{T}$} \pref{thm:gap_lb} shows
that such assumptions are qualitatively required for instance-dependent logarithmic
regret guarantees.

\citet{langford2008epoch} consider a different gap notion we
refer to as a \emph{policy gap} which, in the stochastic setting, posits that
$L(\pistar) < L(\pi)-\Delta_{\mathrm{policy}}$, where
$L(\pi)=\En_{x,\ls}\ls(\pi(x))$. For instances with policy gap
$\Delta_{\mathrm{policy}}$, they show that the Epoch-Greedy algorithm achieves regret
$\poly(\log\abs*{\Pi},\log{}T, \Delta_{\mathrm{policy}}^{-2})$. There is no contradiction
between this result and \pref{thm:gap_lb}, as the construction in
the theorem has policy gap $\frac{1}{\sqrt{T}}$.


\section{Extensions}
\label{sec:extensions}
In this setting we provide two useful generalizations of \mainalg and
its analysis. The first concerns the setting where the realizability
assumption holds only approximately, and the second concerns the
setting where the action set is infinite.
\subsection{Misspecified models}
\label{sec:misspecified}
In practice, the realizability assumption (\pref{ass:realizability})
may be restrictive. In this section we show that the performance
of \mainalg gracefully degrades when the assumption fails to hold, so
long as the learning rate is changed appropriately. We consider the
following relaxed notion of realizability.
\begin{customassumption}{4}
  \label{ass:misspec}
  There exists a regressor $\fstar\in\cF$ such that for all $t$,
$\fstar(x,a) = \En\brk*{\ls_t(a)\mid{}x_t=x} + \veps_{t}(x,a)$, where $\abs*{\veps_{t}(x,a)}\leq{}\veps$.
\end{customassumption}
The main theorem for this section shows that when \pref{ass:misspec}
holds, the performance of \mainalg degrades by an additive
$\veps\cdot\sqrt{K}T$ factor. We state the result in terms of \pref{ass:square_regret}, since this assumption is typically easier to
satisfy than \pref{ass:square_regret2} when the model is misspecified.
\begin{theorem}
  \label{thm:misspec}
  Suppose the adversary satisfies \pref{ass:misspec} and \alg satisfies
  \pref{ass:square_regret}. Then \mainalg with
  $\gamma=2\sqrt{KT/(\RegSquare+2\veps^{2}T)}$ and $\mu=K$ ensures that
  \[
    \sup_{\pi}\En\brk*{
      \sum_{t=1}^{T}\ls_t(a_t) - \sum_{t=1}^{T}\ls_t(\pi(x_t))
      }\leq{} 2\sqrt{KT\cdot{}\RegSquare} + \veps\cdot{}5\sqrt{K}T,
    \]
    where $\sup_{\pi}$ ranges over all policies $\pi:\cX\to\cA$.  
\end{theorem}
\pref{thm:misspec} gives a pseudoregret bound not just to
$\pi_{\fstar}$, but against any policy $\pi:\cX\to\cA$. The theorem also extends to contextual bandits with
an adaptive adversary. The formal setting here is: as follows.\vspace{-5pt}
\begin{itemize}[itemsep=0pt]
\item[] For $t=1,\ldots,T$:
 \begin{itemize}[itemsep=0pt]
 \vspace{-5pt}
 \item Adversary chooses $x_t\in\cX$.
 \item Learner chooses action distribution $p_t$.
 \item Adversary chooses loss $\ls_t:\cA\to\brk*{0,1}$.
 \item Learner samples $a_t\sim{}p_t$ and observes $\ls_t(a_t)$.
 \end{itemize}
\end{itemize}
We make the following assumption on the adversary.
\begin{customassumption}{5}
  \label{ass:misspec2}
The adaptive adversary is constrained such that for every sequence, there exists $\fstar\in\cF$ such that
$\abs*{\ls_{t}(a)-\fstar(x_t,a)}\leq\veps$ for all $a\in\cA$.
\end{customassumption}
The following theorem shows that this assumption suffices to
attain the same guarantee as \pref{thm:misspec}.
\begin{theorem}
  \label{thm:misspec_adversarial}
  Suppose the adversary satisfies \pref{ass:misspec2} and \alg satisfies
  \pref{ass:square_regret}. Then \mainalg with
  $\gamma=\sqrt{8KT/(\RegSquare+\veps^{2}T)}$ and $\mu=K$ ensures that
  \[
    \sup_{\pi}\En\brk*{
      \sum_{t=1}^{T}\ls_t(a_t) - \sum_{t=1}^{T}\ls_t(\pi(x_t))
      }\leq{} \sqrt{2KT\cdot{}\RegSquare} + \veps\cdot\sqrt{2K}T.
    \]
\end{theorem}

\paragraph{Examples and discussion.} Regret bounds for misspecified
linear contextual bandits have recently gathered interest
\citep{van2019comments,lattimore2019learning,neu2020efficient} due to
their connection to reinforcement learning with misspecified linear feature
maps \citep{du2019good}. Consider again the \linucb-type setting where
\begin{equation}
\label{eq:linucb}
  \cF=\crl*{(x,a)\mapsto\tri*{\theta,x_a}\mid{}\theta\in\bbR^{d},\nrm*{\theta}_2\leq{}1},
\end{equation} so that $x_{t}=\prn*{x_{t,a}}_{a\in\cA}$ is a finite collection of
contexts that varies from round to round. By instantiating
\mainalg with the Vovk-Azoury-Warmuth forecaster, which has $\RegSquare\leq{}d\log(T/d)$
\citep{Vovk98,azoury2001relative} and appealing to
\pref{thm:misspec_adversarial}, we get an efficient algorithm with regret
\begin{equation}
\label{eq:linear_misspec}
\bigoht(\sqrt{dKT} + \veps\sqrt{K}T).
\end{equation}
Previous algorithms with similar guarantees either apply only to non-contextual linear bandits \citep{lattimore2019learning}, or attain sub-optimal regret and require additional assumptions when specialized to this setting \citep{neu2020efficient}.\footnote{\cite{neu2020efficient} give an efficient algorithm for a similar but incomparable setting in which contexts are stochastic, but the ($\veps$-approximately linear) Bayes regressor can vary adversarially from round to round. Their algorithm has regret $\bigoht((dK)^{1/3}T^{2/3} + \veps\sqrt{d}T)$.} Interestingly, as remarked by \cite{lattimore2019learning},
the lower bounds of \cite{du2019good} imply that when $K\gg{}d$, the
``price'' of $\veps$-misspecification must grow as
$\Omega(\veps\sqrt{d}T)$. On the other hand, our result shows that the
price can be improved to $\bigoh(\veps\sqrt{K}T)$ in the small-$K$
regime.

\pref{thm:misspec} and \pref{thm:misspec_adversarial} show that we can be robust to misspecification efficiently whenever
online regression is possible efficiently. 
More broadly, these theorems give the first result we are aware of that considers
$\veps$-misspecification for arbitrary classes $\cF$. The theorems imply that $\bigoh(\veps\sqrt{K}T)$ bounds the
price of misspecification for general classes; the complexity of $\cF$
is only reflected in $\RegSquare$.


\subsection{Infinite actions}
\label{sec:infinite}

While the finite-action setting in which \mainalg works is arguably the most basic and fundamental contextual bandit setting, it is desirable to have algorithms that work for large or infinite sets of actions. As a proof of concept, we extend \mainalg to an infinite-action setting where the action space $\cA$
is $\infdim$-dimensional unit $\ls_2$ ball $\Aball$.

For the results in this section, we assume that the regression functions in $\cF$ take the form $f(x,a) = \tri*{f(x),a}$, where $a\in\Aball$ and $f(x)\in\Aball$. We likewise
assume that the predictions of the sub-algorithm $\alg$ have the form
$\yh_t(x,a) = \tri*{\yh_t(x),a}$. Under these assumptions, we design a
variant of \mainalg called \infalg (\pref{alg:infinite}, deferred to
\pref{app:infinite} for space). The regret guarantee for the algorithm is as follows.

\begin{theorem}
  \label{thm:reduction_ball}
Suppose \pref{ass:realizability} and \savehyperref{ass:square_regret}{Assumption 2a/b} hold.
  Then for any $\delta>0$, \infalg with parameter ${\beta =
\sqrt{\infdim{}(\RegSquare+8\log(\delta^{-1}))/T}}$ ensures that with probability at least $1-\delta$,
  \begin{equation}
    \label{eq:reduction_regret_inf}
  \Reg \leq{} 18\sqrt{\infdim{}T\cdot{}\RegSquare} + 90\sqrt{\infdim{}T\log(\delta^{-1})}.
  \end{equation}
\end{theorem}
As a simple example, consider the setting where $\cF=\crl*{(x,a)\mapsto{}\tri*{\theta,a}\mid{}\theta\in\Aball}$; this is a special case of the well-studied infinite-action linear contextual bandit setting \citep{abbasi2011improved} in which the action set is constant across all rounds. For this setting, choosing Vovk-Azoury Warmuth forecaster for \alg gives $\RegSquare\leq{}\infdim\log(T/\infdim)$. By \pref{thm:reduction_ball}, we see that \infalg has $\Reg\leq{}\infdim\sqrt{T\log(T/\infdim)}$ which matches the rate obtained by the OFUL strategy \citep{abbasi2011improved} for this setting.

\begin{algorithm}[t]
  \setstretch{1.1}
  \begin{algorithmic}[1]
    \State \textbf{parameters}:
    \Statex{}~~~~Learning rate $\beta>0$.
    \Statex{}~~~~Online regression oracle $\alg$.
    \For{$t=1,\ldots,T$}
    \State Receive context $x_t$.
    \Statex{}~~~~~\algcomment{Compute oracle's predictions (Eq.\pref{eq:yhat}).}
    \State Compute $\byh_t\ldef\yh_t(x_t,\cdot) \in\bbR^{\infdim}$.
    \State Let $\alpha_t = \frac{\beta}{\nrm*{\byh_t}_2}\wedge\frac{1}{2}$ and $\ytil_t = \frac{\byh_t}{\nrm*{\byh_t}_2}$.
    \State With probability $1-\alpha_t$, set $a_t=-\ytil_t$.
    \State With
      probability $\alpha_t$, select $i\in\brk*{\infdim}$ uniformly at random and set $a_t=\eps\cdot{}e_i$, where $e_i$
      \Statex{}~~~~~~~~is a
      standard basis vector and $\eps\in\pmo$ is a Rademacher random variable.
    \State Observe loss $\ls_t(a_t)$ and update $\alg$ with example $((x_t,a_t),\ls_t(a_t))$.
    \EndFor
  \end{algorithmic}
  \caption{\infalg}
  \label{alg:infinite}
\end{algorithm}

Conceptually, the proof of \pref{thm:reduction_ball} is similar to
\pref{thm:reduction_main}, in the sense that we show that the action
selection scheme---by carefully balancing probability placed on the
best action for $\byh_t$ versus other actions---ensures that the immediate bandit regret is upper
bounded by the immediate square loss regret on a round-by-round
basis. The regret guarantee extends to the misspecified case in the
same fashion in the finite-action case, though the additive error in this
here scales as $\bigoh(\veps\sqrt{\infdim}T)$ rather than $\bigoh(\veps\sqrt{K}T)$.


\section{Discussion}
\label{sec:discussion}

We have presented the first optimal reduction from contextual bandits
to online square loss reduction. Conceptually, we showed that
\emph{online} oracles are a powerful primitive for designing
contextual bandit algorithms, both in terms of computational and
statistical efficiency. Beyond our algorithmic contribution, we have
shed light on the fundamental limits of algorithms for \richcbs, including minimax
rates and gap-dependent regret guarantees. We are hopeful that our
techniques will find broader use,
and that our results will inspire further research on provably
efficient contextual bandits with flexible function
approximation. We outline a few natural directions below.
Going forward, some natural questions are:

\paragraph{Reinforcement learning.} Can the \mainalg strategy be
adapted to give regret bounds for reinforcement learning or continuous
control with unknown dynamics and function approximation (where
$\cF$ is either a class of dynamics models or value functions)? A key challenge here is that
\mainalg is not optimistic, which seems to play a more important role
in the full RL setting.

\paragraph{Adaptivity.}
Compared to full-information learning, adaptive and data-dependent
guarantees are notoriously difficult to come by for contextual bandits
\citep{agarwal2017open,allen2018make,foster2019model}. Can the
\mainalg strategy or a variant be adapted to better take advantage of
nice data?

\paragraph{Further technical directions.}
There are likely many concrete classes of interest for which \mainalg
can be instantiated to give provable guarantees beyond those in
\pref{sec:algorithm}. On the practical side, it would be useful to
extend the continuous action variant of our algorithm from
\pref{sec:infinite} to arbitrary action sets; we intend to pursue this
in future work. Finally, we mention that while \mainalg works
whenever we have access to an online regression oracle, it remains
open whether there exists an optimal oracle-based contextual bandit algorithm that
uses either 1) an \emph{offline} regression oracle, or 2) an
online \emph{classification} oracle. Both of these are interesting
problems.

\subsection*{Acknowledgements}
We thank Akshay Krishnamurthy and Haipeng Luo for helpful discussions. We acknowledge the support of ONR award \#N00014-20-1-2336. DF acknowledges the support of NSF TRIPODS award \#1740751.


\bibliography{refs}

\newpage
\onecolumn

\appendix
\section{Basic technical results}
\begin{lemma}[Freedman's inequality (e.g., \cite{agarwal2014taming})]
  \label{lem:freedman}
  
  Let $(Z_t)_{t\leq{T}}$ be a real-valued martingale difference
  sequence adapted to a filtration $\gfilt_{t}$, and let
  $\En_{t}\brk*{\cdot}\ldef\En\brk*{\cdot\mid\gfilt_{t}}$. If
  $\abs*{Z_t}\leq{}R$ almost surely, then for any $\eta\in(0,1/R)$
    it holds that with probability at least $1-\delta$,
    \[
      \sum_{t=1}^{T}Z_t \leq{} \eta\sum_{t=1}^{T}\En_{t-1}\brk*{Z_t^{2}} + \frac{R\log(\delta^{-1})}{\eta}.
    \]
\end{lemma}

\section{Proofs from \pref{sec:algorithm}}
\label{app:algorithm}
\subsection{Proof of \pref{thm:reduction_main}}

\begin{proof}[\pfref{thm:reduction_main}]
  We first appeal to the following lemma, proven at the end of the
  section, which relates the contextual bandit regret and square loss
  regret to their conditional expectation counterparts under our
  assumptions on the data-generating process and oracle model.
  \begin{lemma}
    \label{lem:realizable_conc}
    If \pref{ass:realizability} and
    \savehyperref{ass:square_regret}{Assumption 2a/b} hold, then with
    probability at least $1-\delta$,
    \[
      \Reg \leq{} 
        \sum_{t=1}^{T}\sum_{a\in\cA}p_{t,a}(\fstar(x_t,a)-\fstar(x_t,\pistar(x_t)))
        + 
        \sqrt{2T\log(2\delta^{-1})},
      \]
    and
    \[
\sum_{t=1}^{T}\sum_{a\in\cA}p_{t,a}\prn*{\yh_t(x_t,a_t)-\fstar(x_t,a_t)}^{2}
      \leq{} 2\RegSquare + 16\log(2\delta^{-1}).
    \]
  \end{lemma}
\noindent{}Observe that by \pref{lem:realizable_conc}, we have that with probability at least $1-\delta$,
\begin{align*}
\Reg &\leq{}
       \sum_{t=1}^{T}\sum_{a\in\cA}p_{t,a}\brk*{(\fstar(x_t,a)-\fstar(x_t,\pistar(x_t)))
  - \frac{\gamma}{4}\prn*{\yh_t(x_t,a)-\fstar(x_t,a)}^{2}}\\
  &~~~~+ \frac{\gamma}{2}\RegSquare + 4\gamma\log(2\delta^{-1}) + \sqrt{2T\log(2\delta^{-1})}.
\end{align*}
From here, our goal is to show that
\begin{align*}
\sum_{t=1}^{T}\sum_{a\in\cA}p_{t,a}\brk*{(\fstar(x_t,a)-\fstar(x_t,\pistar(x_t)))
  - \frac{\gamma}{4}\prn*{\yh_t(x_t,a)-\fstar(x_t,a)}^{2}}
  \leq{} \frac{2KT}{\gamma}.
\end{align*}
Let the time $t$ be fixed, and introduce the abbreviations
$\fstar_{a} \ldef{} \fstar(x_t,a)$, $\yh_a \ldef{} \yh_t(x_t,a)$,
$p_a\ldef{}p_{t,a}$, $\astar\ldef{}\pistar(x_t)$, and
$b\ldef{}b_t$ The remainder of the proof is dedicated to proving the following
per-timestep inequality, which will imply the result.
\begin{lemma}
  \label{lem:per_step}
  For any vector $\yh\in\brk*{0,1}^{K}$, the corresponding probability distribution $p\in\Delta(K)$ 
  on \pref{line:prob} of \mainalg ensures that for any vector $\fstar\in\brk*{0,1}^{K}$, we have
  \[
    \sum_{a\in\cA}p_a\brk*{(\fstar_a-\fstar_{\astar})-\frac{\gamma}{4}(\yh_a-\fstar_a)^{2}}
    \leq{} \frac{2K}{\gamma}.
  \]  
\end{lemma}
\begin{proof}[\pfref{lem:per_step}]
  Let $\eta>0$ be a free parameter to be determined at the end of the
  proof, and consider.
  \[
    \sum_{a\in\cA}p_a\brk*{(\fstar_a-\fstar_{\astar})-\eta(\yh_a-\fstar_a)^{2}}.
  \]
  As a first step, we have
  \begin{align*}
    \sum_{a\in\cA}p_a\brk*{(\fstar_a-\fstar_{\astar})-\eta(\yh_a-\fstar_a)^{2}}
    &  = 
      \sum_{a\neq{}\astar}p_a\brk*{(\fstar_a-\fstar_{\astar})-\eta(\yh_a-\fstar_a)^{2}}
      - \eta{}p_{\astar}(\yh_{\astar}-\fstar_{\astar})^{2} \\
    &  = 
      \sum_{a\neq{}\astar}p_a\brk*{(\yh_a-\fstar_{\astar}) + (\fstar_a-\yh_a)-\eta(\yh_a-\fstar_a)^{2}}
      - \eta{}p_{\astar}(\yh_{\astar}-\fstar_{\astar})^{2} \\
    &  \leq{}
      \sum_{a\neq{}\astar}p_a(\yh_a-\fstar_{\astar})
      - \eta{}p_{\astar}(\yh_{\astar}-\fstar_{\astar})^{2}         + \frac{(1-p_{\astar})}{4\eta},
  \end{align*}
  where the last inequality uses that
  $(\fstar_a-\yh_a)-\eta(\yh_a-\fstar_a)^{2}\leq{}\frac{1}{4\eta}$ by
  AM-GM. For the next step, we similarly upper bound as
  \begin{align*}
    &\sum_{a\neq{}\astar}p_a(\yh_a-\fstar_{\astar})
      - \eta{}p_{\astar}(\yh_{\astar}-\fstar_{\astar})^{2}         + \frac{1}{4\eta}\\
    &=\sum_{a\neq{}\astar}p_a(\yh_a-\yh_{\astar})
      + (1-p_{\astar})(\yh_{\astar}-\fstar_{\astar})     - \eta{}p_{\astar}(\yh_{\astar}-\fstar_{\astar})^{2}     + \frac{1}{4\eta}
    \\
    &\leq{}\sum_{a\neq{}\astar}p_a(\yh_a-\yh_{\astar})
      + \frac{(1-p_{\astar})^{2}}{4\eta{}p_{\astar}}         +
      \frac{1}{4\eta}\\
    &\leq{}\sum_{a\neq{}\astar}p_a(\yh_a-\yh_{\astar})
      + \frac{1}{4\eta{}p_{\astar}}         + \frac{1}{4\eta},
  \end{align*}
  where the first inequality again uses AM-GM. Now, for each $a$,
  define $u_{a}=\yh_a-\yh_b\geq{}0$. Then we have
  \begin{align*}
    \sum_{a\neq{}\astar}p_a(\yh_a-\yh_{\astar})
    &= \sum_{a\neq{}\astar}p_a(u_a - u_{\astar})\\
    &= \sum_{a\neq{}\astar}p_au_a - (1-p_{\astar})u_{\astar}\\
    &= \sum_{a\notin\crl*{\astar,b}}p_au_a - (1-p_{\astar})u_{\astar}\\
    &= \sum_{a\neq{}b}p_au_a -u_{\astar} \\
    &= \sum_{a\neq{}b}\frac{u_a}{\mu+\gamma{}u_a} -u_{\astar} \\
    &\leq{} \frac{K-1}{\gamma}
      -u_{\astar}.
  \end{align*}
  It remains to bound
  \begin{align*}
    -u_{\astar} + \frac{1}{4\eta{}p_{\astar}}    +\frac{K-1}{\gamma}   + \frac{1}{4\eta}.
  \end{align*}
  We consider two cases. First, if $\astar=b$, then we have
  $p_{\astar}=1-\sum_{a\neq{}\astar}\frac{1}{\mu+\gamma{}u_a}\geq{}1-\frac{K-1}{\mu}\geq{}\frac{1}{K}$,
  using the choice $\mu=K$. Dropping the negative $-u_{\astar}$ term,
  this leads to an upper bound of
  \[
    \frac{K}{4\eta{}} +\frac{K-1}{\gamma} + \frac{1}{4\eta}.
  \]
  On the other hand, if $\astar\neq{}b$, we have
  \begin{align*}
    -u_{\astar} + \frac{1}{4\eta{}p_{\astar}}=   -u_{\astar} +
    \frac{(\mu + \gamma{}u_{\astar})}{4\eta{}} = \frac{K}{4\eta} + u_{\astar}\prn*{\frac{\gamma}{4\eta}-1}.
  \end{align*}
  By choosing $\eta=\gamma/4$, we again get an upper bound of
  \[
    \frac{K}{4\eta{}} +\frac{K-1}{\gamma} + \frac{1}{4\eta} =
    \frac{2K}{\gamma}.
  \]
\end{proof}
\noindent{}Summing \pref{lem:per_step} across all rounds and putting everything together,
we have
\[
\Reg \leq{} \frac{\gamma}{2}\RegSquare + 4\gamma\log(2\delta^{-1}) + \frac{2KT}{\gamma}+ \sqrt{2T\log(2\delta^{-1})}.
\]
The prescribed choice for $\gamma$ now gives the regret bound in the
theorem statement.
\end{proof}

  \begin{proof}[\pfref{lem:realizable_conc}]
    We define a filtration:
    \begin{equation}
      \gfilt_{t-1}=\sigma((x_1,a_1,\ls_1(a_1)), \ldots,
      (x_{t-1},a_{t-1},\ls_{t-1}(a_{t-1})), x_t).\label{eq:filtration1}
    \end{equation}
    Note that for any $a$, the online regression algorithm's
    prediction $\yh_t(x_t,a)$ is a measurable with respect to
    $\gfilt_{t-1}$. We apply Azuma-Hoeffding, which implies that with
    probability at least $1-\delta$,
    \begin{align*}
      \Reg = \sum_{t=1}^{T}\ls_t(a_t)-\ls_t(\pistar(x_t))
      &\leq{}
        \sum_{t=1}^{T}\En\brk*{\ls_t(a_t)-\ls_t(\pistar(x_t))\mid\gfilt_{t-1}} +
        \sqrt{2T\log(\delta^{-1})} \\
      &=
        \sum_{t=1}^{T}\sum_{a\in\cA}p_{t,a}(\fstar(x_t,a)-\fstar(x_t,\pistar(x_t)))
        + 
        \sqrt{2T\log(\delta^{-1})}.
    \end{align*}
    Now, suppose \pref{ass:square_regret} holds, so that $\alg$ guarantees that with probability
    $1$,
    \[
      \sum_{t=1}^{T}\prn*{\yh_t(x_t,a_t) - \ls_t(a_t)}^{2} -
      \sum_{t=1}^{T}\prn*{\fstar(x_t,a_t) - \ls_t(a_t)}^{2} \leq{}
      \RegSquare
    \]
    Define
    $M_t = \prn*{\yh_t(x_t,a_t) - \ls_t(a_t)}^{2}
    -\prn*{\fstar(x_t,a_t) - \ls_t(a_t)}^{2}$ and
    $Z_t = \En\brk*{M_t\mid{}\gfilt_{t-1}}-M_t$.
    \begin{lemma}
      \label{lem:filtration}
      The following properties hold:
      \begin{itemize}
      \item $\abs*{Z_t}\leq{}1$.
      \item
        $\En\brk*{M_t\mid\gfilt_{t-1}} =
        \En\brk*{\prn*{\yh_t(x_t,a_t)-\fstar(x_t,a_t)}^{2}\mid{}\gfilt_{t-1}}=\sum_{a\in\cA}p_{t,a}\prn*{\yh_t(x_t,a)-\fstar(x_t,a)}^{2}$
      \item
        $\En\brk*{Z_t^{2}\mid\gfilt_{t-1}} \leq{}
        4\En\brk*{M_t\mid\gfilt_{t-1}}$.
      \end{itemize}
    \end{lemma}
\noindent{}We now apply \pref{lem:freedman} with $\eta=1/8$, which implies
    that with probability at least $1-\delta$,
    \[
      \sum_{t=1}^{T}\En\brk*{M_t\mid\gfilt_{t-1}} \leq{}
      \sum_{t=1}^{T}M_t +
      \frac{1}{8}\sum_{t=1}^{T}\En\brk*{Z_t^{2}\mid\gfilt_{t-1}} +
      8\log(\delta^{-1}) \leq{} \RegSquare +
      \frac{1}{2}\sum_{t=1}^{T}\En\brk*{M_t\mid\gfilt_{t-1}} +
      8\log(\delta^{-1}),
    \]
    or in other words, with probability at least $1-\delta$,
    \[
      \sum_{t=1}^{T}\sum_{a\in\cA}p_{t,a}\prn*{\yh_t(x_t,a_t)-\fstar(x_t,a_t)}^{2}
      \leq{} 2\RegSquare + 16\log(\delta^{-1}).
    \]
    Finally, we handle the case where \pref{ass:square_regret2} holds.
    Define $X_t = \prn*{\yh_{t}(x_t,a_t)-f^{\star}(x,a_t)}^{2}$, and
    observe that
    \[
      \En_{t-1}\brk*{\prn*{X_t-\En_{t-1}\brk*{X_t}}^{2}}\leq{}
      \En_{t-1}\brk*{X^2_t}\leq{} \En_{t-1}\brk*{X_t}.
    \]
    We now apply \pref{lem:freedman} with $\eta=1/2$, which implies
        \[
      \sum_{t=1}^{T}\En\brk*{X_t\mid\gfilt_{t-1}} \leq{}
      \sum_{t=1}^{T}X_t +
      \frac{1}{2}\sum_{t=1}^{T}\En\brk*{X_t\mid\gfilt_{t-1}} +
      2\log(\delta^{-1}),
    \]
    and consequently
    \[
      \sum_{t=1}^{T}\sum_{a\in\cA}p_{t,a}\prn*{\yh_t(x_t,a_t)-\fstar(x_t,a_t)}^{2}
      \leq{} 2\RegSquare + 4\log(\delta^{-1}).
    \]
  \end{proof}

  \begin{proof}[\pfref{lem:filtration}]
That $\abs*{Z_t}$ is bounded by $1$ is immediate. For the second
property, we have
\begin{align*}
&\En\brk*{\prn*{\yh_t(x_t,a_t) - \ls_t(a_t)}^{2} -\prn*{\fstar(x_t,a_t) -
                 \ls_t(a_t)}^{2}\mid\gfilt_{t-1}}\\
  &=
    \En\brk*{2(\yh_t(x_t,a_t) - \ls_t(a_t))(\yh_t(x_t,a_t)-f^{\star}(x_t,a_t))
    - (\yh_t(x_t,a_t)-\fstar(x_t,a_t))^{2}
    \mid\gfilt_{t-1}}\\
    &=
    \En\brk*{2 (\yh_t(x_t,a_t)-\fstar(x_t,a_t))^{2}
    - (\yh_t(x_t,a_t)-\fstar(x_t,a_t))^{2}
      \mid\gfilt_{t-1}}\\
  &=    \En\brk*{ (\yh_t(x_t,a_t)-\fstar(x_t,a_t))^{2}
      \mid\gfilt_{t-1}},
\end{align*}
where we have used that $\En\brk*{\ls(a)\mid{}x,a}=\fstar(x,a)$, and
that $\yh_t(x_t,a_t)$ is independent of $\ls_t$ given $\gfilt_{t-1}$. For the third property, we have
\begin{align*}
  \En\brk*{Z^{2}_t\mid\gfilt_{t-1}} &\leq
  \En\brk*{M^{2}_t\mid\gfilt_{t-1}}\\
  &=   \En\brk*{(\yh_t(x_t,a_t)-\fstar(x_t,a_t))^{2}(\yh_t(x_t,a_t)+\fstar(x_t,a_t)-2\ls_t(a_t))^{2}
  \mid\gfilt_{t-1}}\\
  &\leq{} 4\En\brk*{(\yh_t(x_t,a_t)-\fstar(x_t,a_t))^{2}
  \mid\gfilt_{t-1}},
\end{align*}
since $\abs*{\yh_t(x_t,a_t)+\fstar(x_t,a_t)-2\ls_t(a_t)}\leq{}2$.
\end{proof}

\begin{proof}[\pfref{prop:minimax_lb}]
We first make the choice $\yh=0$, so that the value is lower bounded
by
\[
\min_{p\in\Delta_K}\max_{\fstar\in\brk*{0,1}^{K}}\max_{\astar\in\brk*{K}}\En_{a\sim{}p}\brk*{
  \fstar_{a} - \fstar_{\astar} - \frac{\gamma}{4}\prn*{\fstar_a}^{2}}.
\]
For a given action distribution $p$, we choose
$\astar\in\argmin_{a\in\cA}p_a$ and set $\fstar_{\astar}=0$ and
$\fstar_{a}=\frac{2}{\gamma}$ for all $a\neq{}\astar$. The condition that
$\gamma\geq{}2$ ensures that $\fstar\in\brk*{0,1}^{K}$. The value is
now lower bounded by
\[
\sum_{a\neq{}\astar}p_{a}\prn*{\frac{2}{\gamma} -
\frac{\gamma}{4}\prn*{\frac{2}{\gamma}}^{2}}
=(1-p_{\astar})\frac{1}{\gamma}\geq{}(1-1/K)\frac{1}{\gamma},
\]
where we have used that $p_{\astar}\leq{}1/K$.
\end{proof}

\subsection{Details for specific oracles}
\label{app:details}

\subsubsection{Generalized linear models}
\newcommand{\thetastar}{\theta^{\star}}
\newcommand{\thetatil}{\wt{\theta}}

Consider the setting where
$\cF=\crl*{x\mapsto{}\sigma(\tri*{\theta,x})\mid{}\theta\in\Theta}$,
where $\sigma:\brk*{-1,+1}\to\brk*{0,1}$ is a known, non-decreasing
$1$-Lipschitz link function and $\Theta=\crl*{\theta\in\bbR^{d}\mid{}\nrm*{\theta}_2\leq{}1}$. We consider two variants of the \glmtron
algorithm \citep{kakade2011efficient} and show that they enjoy slow
and fast rates for online prediction, respectively. We only sketch the
arguments, as they are fairly standard. We analyze both variants in the following online learning setting:
\begin{itemize}
\item[] For $t=1,\ldots,T$:
  \begin{itemize}
  \item Nature chooses input instance $x_t$.
  \item Algorithm chooses prediction $\yh_t$.
  \item Nature chooses outcome $y_t$.
  \end{itemize}
\end{itemize}
We allow $x_t$ to be selected by an adaptive adversary, but assume
that there exists some $\thetastar$ such that
$\En\brk*{y\mid{}x}=\sigma(\tri*{\thetastar,x})\rdef\fstar(x,a)$. We also assume that
$\nrm*{x_t}_2\leq{}1$ for all $t$. 

The first \glmtron variant we analyze is based on online gradient descent. Define a ``pseudo-gradient''
\[
  g_t(\theta) = 2(\sigma(\tri*{\theta,x_t})-y_t)x_t
\]
Starting from $\theta_1=0$, we update the iterates via
\begin{itemize}
\item $\thetatil_{t+1}\gets{}\theta_t-\eta{}g_t(\theta_t)$,
\item $\theta_{t+1}\gets{}\argmin_{\theta\in\Theta}\nrm*{\theta-\thetatil_{t+1}}^2_{2}$.
\end{itemize}
At each time $t$ we predict using $\yh_t
=\sigma(\tri*{\theta_t,x_t})$.
\begin{proposition}
  \label{prop:glm1}
  By setting $\eta=\frac{1}{\sqrt{T}}$, the strategy above guarantees
  that with probability at least $1-\delta$,
  \[
    \sum_{t=1}^{T}\prn*{\yh_t-\fstar(x_t)}^{2} \leq{} \bigoh(\sqrt{T\log(\delta^{-1})}).
    \]
\end{proposition}
\begin{proof}[\pfref{prop:glm1}]
  Define a filtration.
  \[
    \hfilt_{t-1}=\sigma(x_1,y_1,\ldots,x_{t-1},y_{t-1},x_t),
\]
  Observe since $\sigma$ is Lipschitz and non-decreasing, we have
  \begin{align*}
    \sum_{t=1}^{T}\prn*{\yh_t-\fstar(x_t)}^{2}
    &=
      \sum_{t=1}^{T}\prn*{\sigma(\tri*{\theta_t,x_t})-\sigma(\tri*{\thetastar,x_t})}^{2}\\
    &\leq{}
      \sum_{t=1}^{T}\prn*{\sigma(\tri*{\theta_t,x_t})-\sigma(\tri*{\thetastar,x_t})}\prn*{\tri*{\theta_t,x_t}-\tri*{\thetastar,x_t}}
    \\ &=
         \frac{1}{2}\sum_{t=1}^{T}\En\brk*{\tri*{g_t(\theta_t),\theta_t-\thetastar}\mid{}\hfilt_{t-1}}.
  \end{align*}
By Azuma-Hoeffding, we are guaranteed that with probability at least $1-\delta$,
\[
  \sum_{t=1}^{T}\En\brk*{\tri*{g_t(\theta_t),\theta_t-\thetastar}\mid{}\hfilt_{t-1}}
  \leq{} \sum_{t=1}^{T}\tri*{g_t(\theta_t),\theta_t-\thetastar} + \bigoh(\sqrt{T\log(\delta^{-1})}).
\]
Finally, the standard analysis for online gradient descent \citep{hazan2016introduction} guarantees that
with probability $1$,
\[
  \sum_{t=1}^{T}\tri*{g_t(\theta_t),\theta_t-\thetastar} \leq{} \bigoh(\sqrt{T}).
\]

\end{proof}
\noindent{}We next consider an online Newton variant of \glmtron.
Starting from $\theta_1=0$, and $\Sigma_0=\veps{}I$, we update the iterates via
\begin{itemize}
\item $\Sigma_{t+1}=\Sigma_t+x_tx_t^{\trn}$.
\item $\thetatil_{t+1}\gets{}\theta_t-\eta{}\Sigma_{t+1}^{-1}g_t(\theta_t)$,
\item $\theta_{t+1}\gets{}\argmin_{\theta\in\Theta}\nrm*{\theta-\thetatil_{t+1}}^2_{\Sigma_{t+1}}$.
\end{itemize}
As before, at each time $t$ we predict using $\yh_t
=\sigma(\tri*{\theta_t,x_t})$.
  \newcommand{\csig}{c_{\sigma}}
  \begin{proposition}
    \label{prop:glm2}
  Suppose that $\sigma'\geq{}\csig>0$. Then for an appropriate choice
  of $\eta$ and $\veps$, the algorithm above ensures that with
  probability at least $1-\delta$,
  \begin{align*}
    \sum_{t=1}^{T}\prn*{\yh_t-\fstar(x_t)}^{2}
    \leq{} \bigoht\prn*{
    \frac{d\log{}T+\log(\delta^{-1})}{\csig^2}
    }.
  \end{align*}

\end{proposition}
\begin{proof}[\pfref{prop:glm2}]
  Consider the filtration
  \[
    \hfilt_{t-1}=\sigma(x_1,y_1,\ldots,x_{t-1},y_{t-1},x_t),
\]
Since $\sigma$ is Lipschitz and increasing, we have
  \begin{align*}
    \sum_{t=1}^{T}\prn*{\yh_t-\fstar(x_t)}^{2}
    &=
      \sum_{t=1}^{T}2\prn*{\sigma(\tri*{\theta_t,x_t})-\sigma(\tri*{\thetastar,x_t})}^{2}
    - \prn*{\sigma(\tri*{\theta_t,x_t})-\sigma(\tri*{\thetastar,x_t})}^{2}\\
    &\leq{}
      \sum_{t=1}^{T}2\prn*{\sigma(\tri*{\theta_t,x_t})-\sigma(\tri*{\thetastar,x_t})}\prn*{\tri*{\theta_t,x_t}-\tri*{\thetastar,x_t}}
      - \prn*{\sigma(\tri*{\theta_t,x_t})-\sigma(\tri*{\thetastar,x_t})}^{2}
    \\ &=
         \sum_{t=1}^{T}\En\brk*{\tri*{g_t(\theta_t),\theta_t-\thetastar}\mid{}\hfilt_{t-1}}
         -
         \prn*{\sigma(\tri*{\theta_t,x_t})-\sigma(\tri*{\thetastar,x_t})}^{2}
    \\ &\leq{}
         \sum_{t=1}^{T}\En\brk*{\tri*{g_t(\theta_t),\theta_t-\thetastar}\mid{}\hfilt_{t-1}}
         - \csig^2\tri*{\theta-\thetastar,x_t}^{2}.
  \end{align*}
  Let $X_t = \tri*{g_t(\theta_t),\theta_t-\thetastar}$ and
  $Z_t=\En\brk*{X_t\mid\hfilt_{t-1}}-Z_t$. Note that
  $\abs*{X_t}\leq{}4$, $\abs*{Z_t}\leq{}8$, and
  \[
    \En\brk*{Z_t^2\mid\hfilt_{t-1}}
    \leq \En\brk*{X_t^2\mid\hfilt_{t-1}}\leq{} 4\tri*{\theta_t-\thetastar,x_t}^{2}.
  \]
  Next, using \pref{lem:freedman} with $\eta=\csig^2/8$, we are guaranteed that with
  probability at least $1-\delta$,
  \begin{align*}
    \sum_{t=1}^{T}\En\brk*{\tri*{g_t(\theta_t),\theta_t-\thetastar}\mid{}\hfilt_{t-1}}
    \leq{}
    \sum_{t=1}^{T}\tri*{g_t(\theta_t),\theta_t-\thetastar}
    + \frac{\csig^2}{2}\sum_{t=1}^{T}\tri*{\theta_t-\thetastar,x_t}^{2}
    + \frac{64\log(\delta^{-1})}{\csig^2},
  \end{align*}
  and consequently
  \begin{align*}
    \sum_{t=1}^{T}\prn*{\yh_t-\fstar(x_t)}^{2}
    \leq{} \sum_{t=1}^{T}\tri*{g_t(\theta_t),\theta_t-\thetastar}-
    \frac{\csig^2}{2}\tri*{\theta-\thetastar,x_t}^{2}
      + \frac{64\log(\delta^{-1})}{\csig^2}.
  \end{align*}
  From here, an argument identical to the usual analysis of Online
  Newton Step (with the only difference being that $\Sigma_t$ is
  updated with $x_tx_t^{\trn}$ rather than $g_tg_t^{\trn}$; see \cite{hazan2016introduction}) implies
  that when $\eta$ and $\veps$ are chosen appropriately,
  \begin{align*}
    \sum_{t=1}^{T}\tri*{g_t(\theta_t),\theta_t-\thetastar}-
    \frac{\csig^2}{2}\tri*{\theta-\thetastar,x_t}^{2}
    &\leq{} \bigoh\prn*{
      \frac{1}{\csig^2}\sum_{t=1}^{T}\nrm*{g_t}_{\Sigma_t}^{2}
    }
    \leq{} \bigoh\prn*{
      \frac{1}{\csig^2}\sum_{t=1}^{T}\nrm*{x_t}_{\Sigma_t}^{2}
      }\leq{} \bigoh\prn*{
          \frac{d\log{}T}{\csig^2}
    }.
  \end{align*}

\end{proof}


\section{Proofs from \pref{sec:minimax}}
\label{app:minimax}
\subsection{Proof of \pref{thm:lb_stochastic}}
\newcommand{\fat}{\mathrm{fat}}
\begin{proof}[\pfref{thm:lb_stochastic}]
  Recall that $\cG\subseteq(\cX\to\brk*{0,1})$ is said to shatter $x_1,\ldots,x_n$ at
scale $\gamma$ if there exists a sequence $s_1,\ldots,s_n\in\brk*{0,1}$ such that
\[
\forall{}\eps\in\pmo^{n}\;\;\exists{}g\in\cG\;\;\text{such
  that}\;\;\eps_t\cdot{}(g(x_t)-s_t) \geq{}\frac{\gamma}{2}\;\;\forall{}t.
\]
The fat-shattering dimension
\citep{alon1997scale,bartlett1998prediction} is then defined via
\[
\fat_{\gamma}(\cG) =
\max\crl*{n\mid{}\text{$\exists{}x_1,\ldots,x_n$ such that $\cG$
    $\gamma$-shatters $x_1,\ldots,x_n$}}.
\]
 We first invoke the following lemma.
  \begin{lemma}[\citet{mendelson2003entropy}]
\label{lem:rudelson}
There exist constants $0<c<1$ and $C\geq{}0$ such that
for all $\veps\in(0,1)$,
\[
\Hiid(\cG,\veps)\leq{}C\cdot\fat_{c\veps}(\cG)\log(1/\veps).
\]
\end{lemma}
Let $\gamma>0$ be fixed. Then
\pref{lem:rudelson}, along with our assumption on the entropy growth
implies that $\fat_{\gamma}(\cG)\geq{}\wt{\Omega}(\gamma^{-p})$. In
particular, we are guaranteed that there exists a set of distinct examples
$x\ind{1},\ldots,x\ind{m}$ and shifts $s\ind{1},\ldots,s\ind{m}\in\brk*{0,1}$ with $m=\wt{\Theta}(\gamma^{-p})$ such that

\begin{equation}
\forall{}\eps\in\pmo^{m}\;\;\exists{}g\in\cG\;\;\text{such
  that}\;\;\eps_t\cdot{}(g(x\ind{t})-s\ind{t})
\geq{}\frac{\gamma}{2}\;\;\forall{}t.\label{eq:shattered_m}
\end{equation}
For each sign pattern $\eps\in\pmo^{m}$, let $g$ be a
function that shatters the sequence in the sense of
\pref{eq:shattered_m}. We let $g_{\eps}$ be the same function
everywhere except on $x\ind{1},\ldots,x\ind{m}$, where it is truncated so that
shattering holds with equality, i.e.
\begin{equation}
\forall{}\eps\in\pmo^{m}:\;\;\eps_t\cdot{}(g_{\eps}(x\ind{t})-s\ind{t})
=\frac{\gamma}{2}\;\;\forall{}t.\label{eq:shattered_exact}
\end{equation}
Define $s:\cX\to\brk*{0,1}$ to be the function that has $s(x\ind{t})=s\ind{t}$ for each $t$ and is arbitrary elsewhere. Now, for each $\eps$, define\footnote{This clipping step is only required because we work with Bernoulli instances in our lower bound, so as to satisfy the constraint $\ls_t\in\brk*{0,1}$. If subgaussian losses are acceptable this step can be removed.}
\[
h_{\eps}(x) = \mathrm{clip}_{\brk{0,1}}\prn*{g_{\eps}(x) - s(x) + \frac{1}{2}},
\]
where 
\[
\mathrm{clip}_{\brk{0,1}}(y)\ldef{}\left\{\begin{array}{ll}
0,&\quad{}y<0,\\
y,&\quad{}y\in\brk*{0,1},\\
1,&\quad{}y>1.
\end{array}
\right.
\]
This function has $h_{\eps}(x\ind{t})=\frac{1}{2}+\eps_{t}\frac{\gamma}{2}$ and $h_{\eps}(x)\in\brk*{0,1}$ for all $x\in\cX$. We form a class $\cH$ by taking one such function $h_{\eps}$ for each
$\eps\in\pmo^{m}$, as well as the constant $\frac{1}{2}$ function. We take
our augmented function class to be $\cG'=\cG\cup\cH$. Since
$\abs*{\cH}\leq{}2^{m}+1$, we have $\Hiid(\cG',\gamma)
\leq{}\Hiid(\cG,\gamma)+m=\wt{\bigoh}\prn*{\gamma^{-p}}$, so the metric entropy
is preserved as desired.

We now construct a collection of problem instances. We choose $\mu$ to
be the uniform distribution over $x_1,\ldots,x_m$, and for each
$\eps\in\pmo^{m}$, we define a corresponding Bayes regression function
\[
  \fstar_{\eps}(x,1)\ldef{}h_{\eps}(x),\quad\text{and}\quad
  \fstar_{\eps}(x,2)\ldef{}\frac{1}{2}.
\]
Finally, for each $a\in\cA=\crl*{\aone,\atwo}$, we take $\ls(a)$ to be a
  Bernoulli random variable with mean $\fstar(x,a)$ conditioned on $x$.

We define a distribution $\cD$ over problem instances by choosing $\eps$
uniformly at random. Since \pref{eq:shattered_exact} holds with
equality, the mean reward functions on different shattered examples $x\ind{i}$ and
$x\ind{j}$ are independent under $\cD$. In particular, we may treat
data as generated via the following process
\begin{itemize}
\item Sample $x_1,\ldots,x_T$ i.i.d. from $\mu$, and let
  $S_i=\crl*{t\mid{}x_t=x\ind{i}}$.
\item For each $i\in\brk*{m}$, independently sample a Bernoulli multi-armed
  bandit instance $\cP_i$ with arm means $\mu(\aone)$ and $\mu(\atwo)$, such that with
  probability $1/2$,
  \[
    \mu(\aone) = s\ind{i}+\frac{\gamma}{2},\quad\text{and}\quad \mu(\atwo) = \frac{1}{2},
  \]
  and otherwise
  \[
    \mu(\aone) = s\ind{i}-\frac{\gamma}{2},\quad\text{and}\quad \mu(\atwo) = \frac{1}{2}.
  \]
\end{itemize}
Note that $\crl*{S_i}_{i\in\brk*{m}}$ and the instances
$\crl*{\cP_i}_{i\in\brk*{m}}$ are independent under this process. For
a given instance $\cP_i$, let $\astar_i$ denote the optimal arm
for this instance. Now,
note that regret decomposes as
\[
  \Reg = \sum_{i\in\brk*{M}}\sum_{t\in{}S_i}\ls_t(a_t) -\ls_t(\astar_i).
\]
Let $N_i=\abs{S_i}$. By Theorem 5.7 of \cite{kleinberg2013bandits},\footnote{See also Theorem A.2 of \cite{auer2002non}.} we have that for any $\gamma<1/12$, for each $i\in\brk*{m}$
\begin{align*}
  \En_{\cP_i}\En_{\crl*{\ls_t}_{t\in{}S_i}}\brk*{\sum_{t\in{}S_i}\ls_t(a_t)
  -\ls_t(\astar_i)}
  \geq{} \frac{\gamma}{60}\cdot{}N_i\indic\crl*{N_i<\frac{1}{64}\gamma^{-2}},
\end{align*}
where $\En_{\cP_i}$ denotes the draw of the instance $\cP_i$
itself. Now, note that $\En\brk*{N_i}=T/m$ and
$\En\brk*{N_i^{2}}\leq{}(T/m)^{2}$. Hence, by Markov's inequality and
Paley-Zygmund, there exist constants $c_1<c_2$ such that with constant
probability (say, $1/8$), we have
$c_1\frac{T}{m}<N_i<c_2\frac{T}{m}$. It follows that if we select $\gamma$ such
that $\frac{T}{m} = c\gamma^{-2}$ for a sufficiently small constant
$c$, i.e. $\gamma\propto{}T^{-\frac{1}{2+p}}/\mathrm{polylog}(T)$, then we are guaranteed
that with probability at least $1/8$ over the draw of $S_i$,
\[
\En_{\cP_i}\En_{\crl*{\ls_t}_{t\in{}S_i}}\brk*{\sum_{t\in{}S_i}\ls_t(a_t)
  -\ls_t(\astar_i)}
  \geq{} \Omega\prn*{\gamma\cdot\frac{T}{m}}.
\]
In particular, since the expected regret is non-negative, we have
\[
\En_{S_i}\En_{\cP_i}\En_{\crl*{\ls_t}_{t\in{}S_i}}\brk*{\sum_{t\in{}S_i}\ls_t(a_t)
  -\ls_t(\astar_i)}
\geq{} \Omega\prn*{\gamma\cdot\frac{T}{m}}.
\]
To conclude, we sum this bound over all $i\in\brk*{m}$ which, by
linearity of expectation, gives
\[
\En\brk*{\Reg} \geq{} \Omega\prn*{\gamma\cdot{}T} = \wt{\Omega}(T^{\frac{1+p}{2+p}}).
\]
\end{proof}

\subsection{Proof of  \pref{thm:universal_ub_tensor}}
To prove the upper bound, we exhibit an online regression algorithm $\alg$ for which $\RegSquare\leq{}
\bigoht\prn*{T^{\frac{1+p}{2+p}}}$ under i.i.d. contexts and the
realizability assumption. The result then follows by appealing to
\pref{thm:reduction_main}.

To describe the algorithm, we introduce additional notation. For a dataset $S=x_1,\ldots,x_n$, we let
$d_{S}(g,g')=\prn*{\frac{1}{n}\sum_{t=1}^{n}\prn*{g(x_t)-g'(x_t)}^{2}}^{1/2}$
denote the empirical $L_2$ distance on $S$. Our strategy for \alg is given below. The algorithm proceeds in
epochs of doubling length. At the beginning of each epoch, the
algorithm forms a cover for $\cG$ using all the data collected so
far. For the remainder of the epoch, it runs a variant of the
exponential weights algorithm (Vovk's aggregating algorithm) over the
cover to produce the square loss predictions $\prn*{\yh_t}$. In more detail, the algorithm is as follows:
\begin{itemize}
\item Parameters: $\veps>0$.
\item Let $M=\ceil*{\log{}T}$, and let $\tau_m=e^{m-1}\wedge{}T$ for
  each $m\in\brk*{M+1}$.
\item For each epoch $m=1,\ldots,M$:
  \begin{itemize}
  \item Let $S_{m}=x_1,\ldots,x_{\tau_m-1}$
  \item Let $\wh{\cG}_{m}$ be a minimal $L_2$ cover for $\cG$ on $S_m$,
    at scale $\veps$.
  \item Let
    $\wh{\cF}_{m} =
    \crl*{(x,a)\mapsto{}g_a(x)\mid{}g\in\wh{\cG}_{m}}$.
  \item Select predictions $\yh_t$ for rounds
    $t=\tau_m,\ldots,\tau_{m+1}-1$ by running the exponential weights algorithm over $\wh{\cF}_m$ with the loss function
    $\mb{\ls}_t(\yh) \ldef{} \prn*{\yh-\ls_t(a_t)}^{2}$, as described in \pref{lem:hedge_regret}.
  \end{itemize}
\end{itemize}
Let $\cI_m = \tau_{m},\ldots,\tau_{m+1}-1$, and let $n_m = \abs*{\cI_m}$. To analyze the performance of the algorithm, we require two standard
lemmas. The first lemma shows that the aggregating algorithm has
regret $\log\abs*{\cF}$ for a finite class.
\begin{lemma}[\citet{PLG}, Proposition 3.2]
  \label{lem:hedge_regret}
  Let $\cF\subseteq(\cZ\to\brk*{0,1})$ be a finite function
  class. Consider an exponential weights algorithm variant with loss
  $\mb{\ls}_t(\yh)\ldef{}\prn*{\yh-y_t}^{2}$, uniform prior and
  learning rate $\eta=1/2$, which follows the update rule
  \begin{align*}
&\textbf{1)}\quad \text{Select}\;\; P_t(g)
\propto{}e^{-\eta\sum_{i=1}^{t-1}(g(z_i)-y_i)^{2}}.\\
&\textbf{2)}\quad \text{Choose $\yh_t$ such that}\;\;
(\yh_t-y)^{2} \leq{} 
-\frac{1}{\eta}\log\prn*{\En_{g\sim{}P_t}e^{-\eta\prn*{g(z_t)-y}^{2}}}\quad\forall{}y\in\brk*{0,1}.
\label{eq:2}
\end{align*}
This strategy guarantees that for any adaptively
  chosen sequence $(z_1,y_1),\ldots,(z_T,y_T)$
  \begin{equation}
    \sum_{t=1}^{T}(\yh_t-y_t)^{2} -
    \min_{f\in\cF}\sum_{t=1}^{T}(f(z_t)-y_t) \leq{}
    2\log\abs*{\cF}.
  \end{equation}
\end{lemma}
\noindent{}The second lemma quantifies the rate at which empirical $L_2$ distance
concentrate around their population counterparts, and is used to prove
that the covers for $\cF$ the algorithm forms at the beginning of
epoch are also accurate on future examples under the \iid assumption.
\begin{lemma}[\citet{rakhlin2017empirical}]
  \label{lem:emp_cover}
  Let $\cP\in\Delta(\cX)$, and let
  $S=x_1,\ldots,x_n$, where $x_t\sim\cP$ i.i.d. for
  all $t$. If $\cG$ has range $\brk*{0,1}$ and
  $\Hiid(\cG,\veps)\propto\veps^{-p}$, then with probability at least
  $1-\delta$,
  \begin{equation}
    \label{eq:emp_cover}
    \En_{x\sim\cP}\prn*{g(x)-g'(x)}^{2} \leq{}
    2{}d_{S}^{2}(g,g') + \bigoh\prn*{
      \log^{3}n\cdot{}n^{-\prn*{1\wedge\frac{2}{p}}} + \log(\log{}n/\delta)\cdot{}n^{-1}
      }, \;\;\text{for all $g,g'\in\cG$.}
  \end{equation}
\end{lemma}
\noindent{}The last auxiliary result we require is a basic concentration lemma for the square loss.
\begin{lemma}
  \label{lem:single_function_conc}
  Let the epoch $m\in\brk*{M}$ and a function $f\in\cF$ be fixed. Then
  with probability at least $1-\delta$,
  \[
    \sum_{t\in\cI_m}\bls_t(\yh_t) - \bls_t(\fstar(z_t))
    \leq{} \sum_{t\in\cI_m}\bls_t(\yh_t) - \bls_t(f(z_t))
    + 2 Kn_m\cdot{}\max_{a\in\cA}\En_{x\sim{}\mu}\brk*{\prn*{f(x,a)-\fstar(x,a)}^{2}}
    + 16\log(\delta^{-1}).
  \]  
\end{lemma}
\begin{proof}[\pfref{lem:single_function_conc}]
  To begin, we have
\begin{align*}
  &\sum_{t=\in\cI_m}\bls_t(\yh_t) - \bls_t(\fstar(z_t)) \\
&  = \sum_{t=\in\cI_m}\bls_t(\yh_t) -\bls_t(f(z_t))
                                                            +    \sum_{t=\in\cI_m}\bls_t(f(z_t))-\bls_t(\fstar(z_t)).
\end{align*}
  It remains to bound the second term.
Define a filtration
  \[
  \hfilt_{t-1}=\sigma((x_1,a_1,\ls_1(a_1)), \ldots,
  (x_{t-1},a_{t-1},\ls_{t-1}(a_{t-1}))),
\]
and let $M_t=\bls_t(f(z_t))-\bls_t(\fstar(z_t))$ and
$Z_t=\En\brk*{M_t\mid\hfilt_{t-1}}-Z_t$. The following lemma shows that these
random variables are both bounded and self-bounding.
  \begin{lemma}
    The following properties hold:
    \begin{itemize}
    \item $\abs*{Z_t}\leq{}1$.
    \item $\En\brk*{M_t\mid\hfilt_{t-1}} =
      \En_{x_t\sim{}\mu}\brk*{\En_{a_t}\brk*{\prn*{f(x_t,a_t)-\fstar(x_t,a_t)}^{2}\mid{}\hfilt_{t-1},x_t}}$
    \item $\En\brk*{Z_t^{2}\mid\hfilt_{t-1}} \leq{} 4\En\brk*{M_t\mid\hfilt_{t-1}}$.
    \end{itemize}
  \end{lemma}
  \begin{proof}
    See proof of \pref{lem:filtration}.
  \end{proof}
\noindent{}We now apply \pref{lem:freedman} with $\eta=1/8$, which implies that
with probability at least $1-\delta$,
\[
  \sum_{t\in\cI_m}\En\brk*{M_t\mid\hfilt_{t-1}} \leq{} \sum_{t\in\cI_m}M_t
  + \frac{1}{8}\sum_{t\in\cI_m}\En\brk*{Z_t^{2}\mid\hfilt_{t-1}} + 8\log(\delta^{-1})
  \leq{} \sum_{t\in\cI_m}M_t
  + \frac{1}{2}\sum_{t\in\cI_m}\En\brk*{M_t\mid\hfilt_{t-1}} + 8\log(\delta^{-1}),
\]
or, rearranging, with probability at least $1-\delta$, we have
\[
  \sum_{t\in\cI_m}\bls_t(f(z_t))-\bls_t(\fstar(z_t))
  \leq{} 2 \sum_{t\in\cI_m}\En\brk*{M_t\mid\hfilt_{t-1}} + 16\log(\delta^{-1}).
\]
The result follows by observing that
\begin{align*}
 \En\brk*{M_t\mid\hfilt_{t-1}} &= \En_{x_t\sim{}\mu}\brk*{\En_{a_t}\brk*{\prn*{f(x_t,a_t)-\fstar(x_t,a_t)}^{2}\mid{}\hfilt_{t-1},x_t}}
  \\
  &\leq{}\En_{x\sim{}\mu}\brk*{\max_{a\in\cA}\brk*{\prn*{f(x,a)-\fstar(x,a)}^{2}}}\\
  &\leq{}K\cdot{}\max_{a\in\cA}\En_{x\sim{}\mu}\brk*{\prn*{f(x,a)-\fstar(x,a)}^{2}}.
\end{align*}

\end{proof}

\begin{proof}[\pfref{thm:universal_ub_tensor}]
  \newcommand{\Gcov}{\wh{\cG}}
  \newcommand{\ghat}{\wh{g}}

The strategy for this proof is as follows. We first show that the
choice for $\alg$ described above enjoys an upper bound on $\RegSquare$ with high probability. We then appeal
to \pref{thm:reduction_main} and tune the parameters to achieve the
final bound.

We will prove a high-probability regret for each epoch $m$ separately, then
union bound and add these regret bounds to get the final bound on
$\RegSquare$. Let the epoch $m\in\brk*{M}$ be fixed. Let $\gstar_a$ be
such that
\[
\fstar(x,a) = \gstar_a(x).
\]
By the definition of $\Gcov_m$, for each $a$ there exists
$\ghat_a\in\Gcov_m$ such that
$d_{S_m}(\ghat_a,\gstar_a)\leq{}\veps$. We apply
\pref{lem:emp_cover}, which implies that with probability at least
$1-\delta$, for all $a\in\cA$,
\begin{equation}
    \En_{x\sim\cP}\prn*{\ghat_a(x)-\gstar_a(x)}^{2} \leq{}
    2{}d_{S}^{2}(\ghat_a,\gstar_a) + \bigoh\prn*{
      \log^{3}T\cdot{}\abs*{S_m}^{-\prn*{1\wedge\frac{2}{p}}} + \log(\log{}T/\delta)\cdot \abs*{S_m}^{-1}
    }.\label{eq:g_conc}
  \end{equation}
Let $\fhat(x,a)=\ghat_a(x)$ denote the corresponding regression function in $\wh{\cF}_m$.
Since contexts are i.i.d.---in particular, contexts in $\cI_m$ are
independent of those in $S_m$---we have by
\pref{lem:single_function_conc} that conditioned on the event in
\pref{eq:g_conc}, with probability at least $1-\delta$,
  \[
    \sum_{t\in\cI_m}\bls_t(\yh_t) - \bls_t(\fstar(z_t))
    \leq{} \underbrace{\sum_{t\in\cI_m}\bls_t(\yh_t) -
      \bls_t(\fhat(z_t))}_{\text{regret to $\fhat$}}
    + 2
    Kn_m\cdot{}\max_{a\in\cA}\underbrace{\En_{x\sim{}\mu}\brk*{\prn*{\fhat(x,a)-\fstar(x,a)}^{2}}}_{\text{bias
      of $\fhat$}}
    + 16\log(\delta^{-1}).
  \]
Henceforth, condition on the event (denoted $\cE_m$) that both this
inequality and \pref{eq:g_conc} hold, which occurs with probability at
least $1-2\delta$ by the union bound. Since \pref{eq:g_conc} holds, we
may bound the bias term as
\begin{align*}
&  Kn_m\cdot{}\max_{a\in\cA}\En_{x\sim{}\mu}\brk*{\prn*{\fhat(x,a)-\fstar(x,a)}^{2}}\\
  & =
    Kn_m\cdot{}\max_{a\in\cA}\En_{x\sim{}\mu}\brk*{\prn*{\ghat_a(x)-\gstar_a(x)}^{2}}\\
  &\leq{}
\bigoht\prn*{
    Kn_m\cdot{}\veps^{2} + Kn_m^{1-\frac{2}{p}} + K\log(\delta^{-1})
    },
\end{align*}
where we have used that the exponential epoch schedule ensures $\abs*{S_m}\geq{}e^{-1}\cdot{}n_m$.

For the regret term we use \pref{lem:hedge_regret} which, since
$\fhat\in\wh{\cF}_m$, implies
\[
\sum_{t\in\cI_m}\bls_t(\yh_t) -
      \bls_t(\fhat(z_t)) \leq{} 2\log\abs*{\wh{\cF}_m} \leq{} \bigoh\prn*{\veps^{-p}},
    \]
    with probability $1$. Putting both bounds together we have that
    conditioned on $\cE_m$,
    \begin{equation}
      \sum_{t\in\cI_m}\bls_t(\yh_t) - \bls_t(\fstar(z_t))
      \leq{}     \bigoh\prn*{\veps^{-p} + 
    Kn_m\cdot{}\veps^{2} + Kn_m^{1-\frac{2}{p}} + K\log(\delta^{-1})
    }.\label{eq:epoch_regret}
  \end{equation}
We now union bound over the events $\cE_m$ for all $m\in\brk*{M}$ and
sum the bound \pref{eq:epoch_regret} over each round. Since there are
at most $\log{}T+1$ epochs, we are guaranteed (after simplifying)
that with probability at least $1-\delta$,
\[
  \RegSquare = \sum_{t=1}^{T}\bls_t(\yh_t) - \bls_t(\fstar(z_t))
  \leq{} \bigoht\prn*{\veps^{-p} + 
    KT\cdot{}\veps^{2} + KT^{1-\frac{2}{p}} + K\log(\delta^{-1}).
    }.
  \]
  Choosing $\veps\propto(KT)^{-\frac{1}{2+p}}$ ensures that
  $\RegSquare\leq{}\bigoht\prn*{(KT)^{1-\frac{2}{2+p}}+KT^{1-\frac{2}{p}}+K\log(\delta^{-1})}$. However,
  note that $KT^{1-\frac{2}{p}}\leq{}(KT)^{1-\frac{2}{2+p}}$ whenever
  $K\leq{}T^{\frac{2}{p}}$, and if $K\geq{}T^{\frac{2}{p}}$ then
  $(KT)^{1-\frac{2}{2+p}}\geq{}T$, so the regret bound is vacuous. Hence, we can simplify
  to $\RegSquare\leq{}\bigoht\prn*{(KT)^{1-\frac{2}{2+p}}+K\log(\delta^{-1})}$. To
  finish, we appeal to \pref{thm:reduction_main} and union bound,
  which gives
  \[
    \Reg\leq{}\bigoh\prn*{\sqrt{KT\RegSquare} + \sqrt{K^{2}T\log(\delta^{-1})}}
    = \bigoht\prn*{(KT)^{\frac{1+p}{2+p}} + \sqrt{K^{2}T\log(\delta^{-1})}}.
  \]

  \end{proof}


\section{Proofs from \pref{sec:gap}}
\label{app:gap}

\newcommand{\RegBar}{\overline{\mathrm{Reg}}_{T}}
\begin{proof}[\pfref{thm:gap_lb}]
This proof uses arguments similar to a lower bound against strongly
adaptive regret for (non-contextual) multi-armed bandits in given in \citet{daniely2015strongly}.
  
  Let $N\in\brk*{T}$ be fixed; throughout the proof we assume without
  loss of generality that $T$ is divisible by $N$. We take $\cX=\brk*{N}$, $\cA=\crl*{\aone,\atwo}$ and
  $\cF=\crl*{f_i}_{i=0}^{N}$, where for each $i\leq{}1\leq{}N$, $f_i$ is defined as
  follows:
  \begin{align*}
    f_i(j,\aone) =
    \frac{1}{2}-\Delta,\quad&\text{and}\quad{}f_i(j,\atwo)=\frac{1}{2},\quad\forall{}j\neq{}i.\\
    f_i(i,\aone) = \frac{1}{2}-\Delta,\quad&\text{and}\quad{}f_i(i,\atwo)=\frac{1}{2}-2\Delta.
  \end{align*}
  The regressor $f_0$ is defined via
  \[
    f_0(j,\aone) =
    \frac{1}{2}-\Delta,\quad\text{and}\quad{}f_0(j,\atwo)=\frac{1}{2},\quad\forall{}j.
  \]
For each $i\in\crl*{0,\ldots,N}$ we define a problem instance $\cP_i$ as
follows:
\begin{itemize}
\item Choose $\fstar=f_i$.
\item For each context $x$,  set $\ls(a)=\fstar(x,a)$ with probability $1$
  conditioned on $x$ (i.e., the instances are noiseless).
\item Play $x_t=1$ for rounds $1,\ldots,T/N$, $x_t=2$ for rounds
  $T/N+1,\ldots,2T/N$, and so on.
\end{itemize}
Each instance has uniform gap parameter $\Delta$. Observe that for each instance $\cP_i$ above, we have
\[
  \Reg = \sum_{t=1}^{T}f_i(x_t,a_t) - f_i(x_t,\pi_i(x_t)),
\]
where $\pi_i\ldef{}\pi_{f_i}$ is the optimal policy for $f_i$.

Formally, we model the contextual bandit algorithm $\cb$ as a sequence
of measurable functions $\cb_t:\brk*{0,1}^{t-1}\times\cR\to\cA$, such that
\[
a_t = \cb_t(\ls_1(a_1),\ldots,\ls_{t-1}(a_{t-1})\midsem{}r),
\]
where $r\sim{}P_r$ is a random seed.

Henceforth, let the algorithm $\cb$ be fixed. We consider two cases. The goal will be to show that in each case
there exists some $i$ such that
$\En_{\cP_i}\brk*{\Reg}\geq{}\Delta\frac{T}{N}$.\footnote{Note that since
losses are noiseless, the expectation here only reflects the
randomization over the algorithm's actions under instance $\cP_i$.} We then show that this
quantity grows as $\sqrt{T}$ for an appropriate choice of $N$, even for $\Delta$ constant.
\paragraph{Case 1: $\En_{\cP_0}\brk*{\Reg}>
\Delta\frac{T}{N}$.}
For the first case, we assume $\En_{\cP_0}\brk*{\Reg}>
\Delta\frac{T}{N}$, so the desired statement follows immediately by taking $i=0$.
\paragraph{Case 2: $\En_{\cP_0}\brk*{\Reg}\leq{}
\Delta\frac{T}{N}$.}
For the second case, suppose $\cb$ has
$\En_{\cP_0}\brk*{\Reg}\leq{} \Delta\frac{T}{N}$.  Under $\cP_0$, we
have
\[
  \Reg = \sum_{t=1}^{T}\Delta\cdot\indic\crl*{a_t=\atwo},
\]
so if we define $U=\crl*{t\mid{}a_t=\atwo}$, this implies that
$\En_{\cP_0}\abs*{U}\leq{}\frac{T}{N}$. For each $i$, let $\cI_i$ denote the rounds in which
$x_t=i$. Since these sets form a partition,
i.e. $\cI_1\cup\ldots\cup\cI_N=\brk*{T}$, we have
\[
  \sum_{i=1}^{N}\En_{\cP_0}\abs*{\cI_i\cap{}U}  = \En_{\cP_0}\abs*{U}\leq{}\frac{T}{N},
\]
and in particular, $N\cdot\min_{i}\En_{\cP_0}\abs*{\cI_i\cap{}U}
\leq{} \frac{T}{N}$. If $N=\sqrt{2T}$, this implies that for some index
$\istar$,  $\En_{\cP_0}\abs*{\cI_{\istar}\cap{}U}\leq\frac{1}{2}$, and
consequently
$\bbP_{\cP_0}\prn*{\cI_{\istar}\cap{}U=\emptyset}\geq{}\frac{1}{2}$. We
emphasize that $\istar$ is not a random variable; it is
a deterministic property of $\cb$. To proceed, we use the following lemma, which implies we also have $\bbP_{\cP_{\istar}}\prn*{\cI_{\istar}\cap{}U=\emptyset}\geq{}\frac{1}{2}$.
\begin{lemma}
  \label{lem:instance_relation}
  For any every instance $i$,
  $\bbP_{\cP_0}\prn*{\cI_{i}\cap{}U=\emptyset} =  \bbP_{\cP_i}\prn*{\cI_{i}\cap{}U=\emptyset}$.
\end{lemma}
\noindent{}Observe that under instance $\istar$, we have
\[
  \Reg = \sum_{t=1}^{T}f_{\istar}(x_t,a_t) - f_{\istar}(x_t,\pi_i(x_t))
  \geq{} \sum_{t\in\cI_{\istar}}\Delta\indic\crl*{a_t\neq{}\atwo},
\]
and in particular
$\En_{\cP_{\istar}}\brk*{\Reg}\geq{}\Delta\frac{T}{N}\cdot\bbP_{\cP_{\istar}}\crl*{\cI_{\istar}\cap{}U=\emptyset}$. Since
we showed that
$\bbP_{\cP_{\istar}}\prn*{\cI_{\istar}\cap{}U=\emptyset}\geq{}\frac{1}{2}$,
we conclude that
$\En_{\cP_{\istar}}\brk*{\Reg}\geq{}\frac{1}{2}\Delta\frac{T}{N}$.
\paragraph{Final result.}
Combining the two cases, we have
\[
\max_{i\in\brk*{N}}\En_{\cP_{i}}\brk*{\Reg}\geq{}\frac{1}{2}\Delta\frac{T}{N}.
\]
Recalling that $N=\sqrt{2T}$ and making the somewhat arbitrary choice
$\Delta=\frac{1}{4}$, we have
\[
\max_{i\in\brk*{N}}\En_{\cP_{i}}\brk*{\Reg}\geq{}\frac{1}{16}\sqrt{T}.
\]

\end{proof}
\begin{proof}[\pfref{lem:instance_relation}]
Let $M=T/N$, and let
  $t_1,\ldots,t_M=\crl*{M(i-1)+1,Mi}$ denote the (consecutive) rounds
  in block $\cI_i$. Let $\cH=\ls_1(a_t),\ldots,\ls_{t_1-1}(a_{t_1-1})$
  denote the history prior to $\cI_i$.

Let $\ls^{i}_t(a) = f_i(x_t, a)$ for each $a$. For all $t < t_1$, we
have $\ls^{i}_t(a) = \ls^{0}_t(a)$ (since $f_i(j,a) = f_0(j,a)$ for all $j<i$ and
  losses are noiseless). Consequently, the history is a measurable
  function of $r$ that does not depend on whether the underlying
  instance is $\cP_i$ or $\cP_0$; we will denote the value it take for
  each choice of $r$ by $\cH_r$.

Now, observe that we have
  \[
    \bbP_{\cP_i}\prn*{\cI_{i}\cap{}U=\emptyset} =
    \bbP_{\cP_i}\prn*{a_{t_1}=\aone}\bbP_{\cP_i}\prn*{a_{t_2}=\aone\mid{}a_{t_1}=\aone}\cdots
    \bbP_{\cP_i}\prn*{a_{t_M}=\aone\mid{}a_{t_1},\ldots,a_{t_M-1}=\aone}.
  \]
  We first observe
  that
  $\bbP_{\cP_i}\prn*{a_{t_1}=\aone}=\bbP_{\cP_0}\prn*{a_{t_1}=\aone}$,
  since
  \[
    \bbP_{\cP_i}\prn*{a_{t_1}=\aone}
    = \bbP_{\cP_i}\prn*{\cb_{t_1}(\cH_r\midsem{}r)=\aone}
    = \bbP_{\cP_0}\prn*{\cb_{t_1}(\cH_r\midsem{}r)=\aone}
    =\bbP_{\cP_i}\prn*{a_{t_1}=\aone},
  \]
  where we used the previous observation that the law of $r$ and
  $\cH_r$ does not depend on the underlying instance. For the next
  timestep, we use that
  $\ls_{t_1}^i(\aone)=\ls_{t_1}^0(\aone)=\frac{1}{2}-\Delta$. That
  is, if the algorithm pulls arm $\aone$ at time $t_1$, the feedback
  is identical for both instances. Consequently, we have
  \begin{align*}
    \bbP_{\cP_i}\prn*{a_{t_2}=\aone\mid{}a_{t_1}=\aone}
    &=
      \bbP_{\cP_i}\prn*{\cb_{t_2}(\cH_r,\ls_{t_1}^{i}(\aone)\midsem{}r)=\aone\mid{}a_{t_1}=\aone}
    \\
    &=
      \bbP_{\cP_i}\prn*{\cb_{t_2}(\cH_r,\ls_{t_1}^{0}(\aone)\midsem{}r)=\aone\mid{}a_{t_1}=\aone}\\
    &= \bbP_{\cP_0}\prn*{a_{t_2}=\aone\mid{}a_{t_1}=\aone}.
  \end{align*}
  Proceeding forwards in the same fashion for $t=t_3,\ldots,t_M$ yields the result.
\end{proof}


\section{Proofs from \pref{sec:extensions}}
\label{app:extensions}
\subsection{Proofs from \pref{sec:misspecified}}
\label{app:misspecified}

\begin{proof}[Proof of \pref{thm:misspec}]
  Let a policy $\pi$ be fixed, and let the filtration $\gfilt_t$ be as in \pref{eq:filtration1}. We begin with the sum of (conditional expectations of) immediate regrets of to $\pi$.
  \begin{align*}
    &\sum_{t=1}^{T}\En\brk*{\ls_t(a_t)-\ls_t(\pi(x_t))\mid\gfilt_{t-1}} \\
        &\leq{}\sum_{t=1}^{T}\En\brk*{\fstar(x_t,a_t)-\fstar(x_t,\pi(x_t))\mid\gfilt_{t-1}} + 2\veps{}T\\
    &\leq{}\sum_{t=1}^{T}\En\brk*{\fstar(x_t,a_t)-\fstar(x_t,\pi_{\fstar}(x_t))\mid\gfilt_{t-1}} + 2\veps{}T\\
    &=\sum_{t=1}^{T}\sum_{a\in\cA}p_{t,a}\prn*{\fstar(x_t,a)-\fstar(x_t,\pi_{\fstar}(x_t))} + 2\veps{}T.
  \end{align*}
  Applying \pref{lem:per_step} at each round and summing, we can upper bound by
  \begin{align*}
    &\frac{\gamma}{4} \sum_{t=1}^{T}\sum_{a\in\cA}p_{t,a}\prn*{\yh_{t}(x_t,a) - \fstar(x_t,a)}^{2} + \frac{2KT}{\gamma} + 2\veps{}T \\
    &=\frac{\gamma}{4} \sum_{t=1}^{T}\En\brk*{\prn*{\yh_{t}(x_t,a_t) - \fstar(x_t,a_t)}^{2}\mid\gfilt_{t-1}} + \frac{2KT}{\gamma} + 2\veps{}T.
  \end{align*}
  By the law of total expectation, this implies that the following inequality holds marginally.
  \begin{align}
    \En\brk*{\sum_{t=1}^{T}\ls_t(a_t)-\ls_t(\pi(x_t))}
    \leq{} \frac{\gamma}{4} \En\brk*{\sum_{t=1}^{T}\prn*{\yh_{t}(x_t,a_t) - \fstar(x_t,a_t)}^{2}} + \frac{2KT}{\gamma} + 2\veps{}T.
    \label{eq:misspec_ub}
  \end{align}
  In the remainder of the proof we bound the squared prediction error on the right-hand side. Observe that for any fixed timestep $t$, strong convexity of the square loss implies
  \begin{align*}
    &(\yh_t(x_t,a_t) - \ls_t(a_t))^2\\
    &\geq{} (\fstar(x_t,a_t) - \ls_t(a_t))^2 + 2(\fstar(x_t,a_t)-\ls_t(a_t))(\yh_t(x_t,a_t)-\fstar(x_t,a_t)) + \prn*{\yh_t(x_t,a_t)-\fstar(x_t,a_t)}^2.
  \end{align*}
  Summing this inequality across all rounds, we are guaranteed that for every sequence of outcomes,
  \begin{align}
    \sum_{t=1}^{T}
    \prn*{\yh_t(x_t,a_t)-\fstar(x_t,a_t)}^2
    &\leq{} \sum_{t=1}^{T}(\yh_t(x_t,a_t) - \ls_t(a_t))^2 - (\fstar(x_t,a_t) - \ls_t(a_t))^2\notag\\
    &~~~~-2\sum_{t=1}^{T}(\fstar(x_t,a_t)-\ls_t(a_t))(\yh_t(x_t,a_t)-\fstar(x_t,a_t)).\notag\\
      &\leq{} \RegSquare -2\sum_{t=1}^{T}(\fstar(x_t,a_t)-\ls_t(a_t))(\yh_t(x_t,a_t)-\fstar(x_t,a_t)),\label{eq:misspec_linearized}
  \end{align}
  where we have used \pref{ass:square_regret}. Now, observe that we have
  \begin{align*}
    &-2\sum_{t=1}^{T}\En\brk*{(\fstar(x_t,a_t)-\ls_t(a_t))(\yh_t(x_t,a_t)-\fstar(x_t,a_t))\mid\gfilt_{t-1}}\\
    &=2\sum_{t=1}^{T}\En\brk*{\veps_{t}(x_t,a_t)(\yh_t(x_t,a_t)-\fstar(x_t,a_t))\mid\gfilt_{t-1}},
  \end{align*}
  where we have used that, since the adversary is oblivious, $\fstar$ does not depend on the outcomes $\ls_1,\ldots,\ls_{T}$.  By the AM-GM inequality (specifically, that $ab \leq{}a^{2} + \frac{1}{4}b^2$ for all $a$, $b$), and \pref{ass:misspec2} we have
  \begin{align*}
    2\sum_{t=1}^{T}\En\brk*{\veps_{t}(x_t,a_t)(\yh_t(x_t,a_t)-\fstar(x_t,a_t))\mid\gfilt_{t-1}}
    &\leq{} 2\sum_{t=1}^{T}\veps^2_t(x_t,a_t) + \frac{1}{2}\sum_{t=1}^{T}\En\brk*{(\yh_t(x_t,a_t)-\fstar(x_t,a_t))^{2}\mid\gfilt_{t-1}}\\
    &\leq{} 2\veps^{2}T+ \frac{1}{2}\sum_{t=1}^{T}\En\brk*{(\yh_t(x_t,a_t)-\fstar(x_t,a_t))^{2}\mid\gfilt_{t-1}}.    
  \end{align*}
  Using the law of total expectation, we conclude that
  \[
    -2\En\brk*{\sum_{t=1}^{T}(\fstar(x_t,a_t)-\ls_t(a_t))(\yh_t(x_t,a_t)-\fstar(x_t,a_t))}
    \leq{} 2\veps^{2}T+ \frac{1}{2}\En\brk*{\sum_{t=1}^{T}(\yh_t(x_t,a_t)-\fstar(x_t,a_t))^{2}}.    
  \]
  Combining this inequality with \pref{eq:misspec_linearized}, we have
  \begin{align*}
    \En\brk*{\sum_{t=1}^{T}
    \prn*{\yh_t(x_t,a_t)-\fstar(x_t,a_t)}^2}
    &\leq{} \RegSquare + 2\veps^{2}T+ \frac{1}{2}\En\brk*{\sum_{t=1}^{T}(\yh_t(x_t,a_t)-\fstar(x_t,a_t))^{2}}, 
  \end{align*}
  or, after rearranging,
  \begin{align*}
    \En\brk*{\sum_{t=1}^{T}
    \prn*{\yh_t(x_t,a_t)-\fstar(x_t,a_t)}^2}
    &\leq{} 2\RegSquare + 4\veps^{2}T.
  \end{align*}
  Using this bound with \pref{eq:misspec_ub} gives
\begin{align*}
  \En\brk*{\sum_{t=1}^{T}\ls_t(a_t)-\ls_t(\pi(x_t))}
  \leq{} \frac{\gamma}{2}(\RegSquare+2\veps^{2}T) + \frac{2KT}{\gamma} + 2\veps{}T.
\end{align*}
Finally, we observe that he choice for $\gamma$ in the theorem statement makes the right-hand side above equal to
\[
  2\sqrt{KT(\RegSquare + 2\veps^{2}T)} + 2\veps{}T
  \leq{} 2\sqrt{KT\cdot{}\RegSquare} + 5\veps{}\sqrt{K}T.
\]

\end{proof}

\begin{proof}[\pfref{thm:misspec_adversarial}]
  Let a policy $\pi$ be fixed, and let the filtration $\gfilt_t$ be as in \pref{eq:filtration1}. Observe that we have
  \begin{align*}
    &\sum_{t=1}^{T}\En\brk*{\ls_t(a_t)-\ls_t(\pi(x_t))\mid\gfilt_{t-1}} \\
        &\leq{}\sum_{t=1}^{T}\max_{\astar\in\cA}\En\brk*{\ls_t(a_t)-\ls_t(\astar)\mid\gfilt_{t-1}}\\
    &=\sum_{t=1}^{T}\max_{\astar}\sum_{a\in\cA}p_{t,a}\prn*{\ls_t(a)-\ls_t(\astar)}.
  \end{align*}
  Applying \pref{lem:per_step} at each round and summing, we can upper bound by
  \begin{align*}
    &\frac{\gamma}{4} \sum_{t=1}^{T}\sum_{a\in\cA}p_{t,a}\prn*{\yh_{t}(x_t,a) - \ls_t(a)}^{2} + \frac{2KT}{\gamma}\\
    &=\frac{\gamma}{4} \sum_{t=1}^{T}\En\brk*{\prn*{\yh_{t}(x_t,a_t) - \ls_t(a_t)}^{2}\mid\gfilt_{t-1}} + \frac{2KT}{\gamma}.
  \end{align*}
  By the law of total expectation, this implies
  \begin{align}
    \En\brk*{\sum_{t=1}^{T}\ls_t(a_t)-\ls_t(\pi(x_t))}
    \leq{} \frac{\gamma}{4} \En\brk*{\sum_{t=1}^{T}\prn*{\yh_{t}(x_t,a_t) - \ls_t(a_t)}^{2}} + \frac{2KT}{\gamma}.
    \label{eq:misspec_ub2}
  \end{align}
  Next, we observe that by \pref{ass:square_regret}, we are guaranteed that with probability $1$,
  \begin{align*}
  \sum_{t=1}^{T}\prn*{\yh_{t}(x_t,a_t) - \ls_t(a_t)}^{2} &\leq{} \inf_{f\in\cF}\sum_{t=1}^{T}\prn*{f(x_t,a_t) - \ls_t(a_t)}^{2} + \RegSquare\\
  &\leq{} \veps^{2}T + \RegSquare,
  \end{align*}
  where the second inequality follows from \pref{ass:misspec2}. Since this bound holds pointwise, it holds in expectation in particular, so we can combine with \pref{eq:misspec_ub2} to get
  \begin{align*}
      \En\brk*{\sum_{t=1}^{T}\ls_t(a_t)-\ls_t(\pi(x_t))}
      \leq{} \frac{\gamma}{4}\prn*{\veps^{2}T +\RegSquare} + \frac{2KT}{\gamma}.
  \end{align*}
  The choice for $\gamma$ in the theorem statement makes this at most
  \[
  \sqrt{2KT(\RegSquare+\veps^{2}T)} \leq{}   \sqrt{2KT\cdot{}\RegSquare} + \veps\sqrt{2K}T.
  \]

\end{proof}


\subsection{Proofs from \pref{sec:infinite}}
\label{app:infinite}
\begin{proof}[\pfref{thm:reduction_ball}]
  Let $p_t$ denote the distribution over actions at time $t$,
  conditioned on $\gfilt_{t-1}$ (where $\gfilt_{t-1}$ is defined as in
  \pref{thm:reduction_main}). Using \pref{lem:realizable_conc}, we have that
with probability at least $1-\delta$, 
\[
  \sum_{t=1}^{T}\En_{a_t\sim{}p_t}\brk*{\prn*{\yh_t(x_t,a_t)-\fstar(x_t,a_t)}^{2}}
  \leq{} 2\RegSquare + 16\log(2\delta^{-1}),
\]
and 
\begin{align*}
  \Reg \leq{}
      \sum_{t=1}^{T}\En_{a_t\sim{}p_t}\brk*{\fstar(x_t,a_t)-\fstar(x_t,\pistar(x_t))}
      + 
    \sqrt{2T\log(2\delta^{-1})}.
\end{align*}
Define $\bfstar_t=\fstar(x_t,\cdot)\in\Aball$. We will show that for
an appropriate choice of constants $C_1$ and $C_2$, for each timestep $t$,
\[
  \max_{\astar\in\Aball}\En_{a_t\sim{}p_t}\brk*{\tri*{\bfstar_t,a_t}-\tri*{\bfstar_t,\astar}}
  \leq{} C_1\cdot
  \En_{a_t\sim{}p_t}\brk*{\prn*{\tri*{\byh_t,a_t} - \tri*{\bfstar_t,a_t}}^{2}} + C_2,
\]
from which it will follow that
\begin{equation}
  \label{eq:regret_generic}
  \Reg \leq{} 2C_1\RegSquare + 16C_1\log(\delta^{-1}) + C_2T + \sqrt{2T\log(\delta^{-1})}.
\end{equation}
\paragraph{Basic properties of the action distribution.}
We begin by calculating the first and second moment of the action distribution.
\begin{align}
  &\mu_t\ldef{}\En_{a_t\sim{}p_t}\brk*{a_t} = (1-\alpha_t)\cdot(-\ytil_t)
    + \alpha_t\cdot\En_{i,\eps}\brk*{\eps\cdot{}e_i} =
    -(1-\alpha_t)\ytil_t.\label{eq:pt_mean}
  \\
  &\Sigma_t\ldef{}\En_{a_t\sim{}p_t}\brk*{a_ta_t^{\trn}} = (1-\alpha_t)\cdot{}\ytil_t\ytil_t^{\trn}
    +
    \alpha_t\cdot\En_{i,\eps}\brk*{\eps^{2}\cdot{}e_ie_i^{\trn}}
    =
    (1-\alpha_t)\ytil_t\ytil_t^{\trn} + \alpha_t\frac{1}{\infdim{}}I \geq{} \frac{\alpha_t}{\infdim{}}\cdot{}I\label{eq:pt_cov}.
\end{align}
Note that with this notation, we have
\begin{equation}
\En_{a_t\sim{}p_t}\brk*{\prn*{\tri*{\byh_t,a_t} -
    \tri*{\bfstar_t,a_t}}^{2}}
= \nrm*{\byh_t-\bfstar}_{\Sigma_t}^{2},\label{eq:square_cov}
\end{equation}
where $\nrm*{x}_A=\sqrt{\tri*{x,Ax}}$ denotes the weighted euclidean norm.

From here, we break the analysis into two cases based on the value of $\alpha_t$.
\paragraph{Case 1: $\alpha_t=\frac{1}{2}$.}
This constitutes a degenerate case in which $\byh_t$ is very small. Indeed, we have
$\frac{\beta}{\nrm*{\byh_t}_2}\geq{}\frac{1}{2}$, so
$\nrm*{\byh_t}_2\leq{}2\beta$. Moreover, by \pref{eq:pt_cov} we have
$\Sigma_t \psdgeq\frac{1}{2\infdim{}}I$. We upper
bound the instantaneous regret by Cauchy-Schwarz and the AM-GM inequality as
\begin{align*}
  \max_{\astar\in\Aball}\En_{a_t\sim{}p_t}\brk*{\tri*{\bfstar_t,a_t}-\tri*{\bfstar_t,\astar}}
  \leq{} 2\nrm*{\bfstar_t} \leq{} \frac{1}{\eta_1} + \eta_1\nrm*{\bfstar_t}_{2}^{2},
\end{align*}
where $\eta_1>0$ is a free parameter. Using
our bound on $\nrm*{\byh_t}_2$ and our lower bound on $\Sigma_t$, we have
\begin{align*}
  \nrm*{\bfstar_t}_{2}^{2}
  \leq{} 2\nrm*{\byh_t-\bfstar_t}_{2}^{2} + 2\nrm*{\byh_t}_2^{2}
  \leq{} 2\nrm*{\byh_t-\bfstar_t}_{2}^{2} + 8\beta^{2}
  \leq{} 4\infdim{}\nrm*{\byh_t-\bfstar_t}_{\Sigma_t}^{2} + 8\beta^{2}.
\end{align*}
Hence, we have
\[
  \max_{\astar\in\Aball}\En_{a_t\sim{}p_t}\brk*{\tri*{\bfstar_t,a_t}-\tri*{\bfstar_t,\astar}}
  \leq{} \frac{1}{\eta_1} + 8\beta^{2}\eta_1 + 4\infdim{}\eta_1\cdot\nrm*{\byh_t-\bfstar_t}_{\Sigma_t}^{2}.
\]
We choose $\eta_1=\frac{1}{\beta}$, which gives
\begin{equation}
  \label{eq:cont_case1}
  \max_{\astar\in\Aball}\En_{a_t\sim{}p_t}\brk*{\tri*{\bfstar_t,a_t}-\tri*{\bfstar_t,\astar}}
  \leq{} 9\beta + \frac{4\infdim{}}{\beta}\cdot\nrm*{\byh_t-\bfstar_t}_{\Sigma_t}^{2}.
\end{equation}

\paragraph{Case 2: $\alpha_t = \frac{\beta}{\nrm*{\byh_t}_2}$.}
This case constitutes the interesting part of the proof. Define $\ftil_t=\bfstar_t/\nrm*{\bfstar}_2$. To begin we let the $\astar_t$ denote the optimal action and write.
\begin{align*}
  \max_{\astar\in\Aball}\En_{a_t\sim{}p_t}\brk*{\tri*{\bfstar_t,a_t}-\tri*{\bfstar_t,\astar}}
  &= \tri*{\mu_t-\astar_t,\bfstar_t}\\
  &= \tri*{\mu_t-\astar_t,\byh_t} + \tri*{\mu_t-\astar_t,\bfstar-\byh_t}.
\end{align*}
Noting that $\astar=-\ftil_t$ and using \pref{eq:pt_mean}, the first term is bounded as
\begin{align*}
  \tri*{\mu_t-\astar_t,\byh_t} = \tri*{\ftil_t-(1-\alpha_t)\ytil_t,\byh_t}
  & = \alpha_t\nrm*{\byh_t} + \underbrace{\tri*{\byh_t,\ftil_t}-\nrm*{\byh_t}}_{\rdef{}-\cE_t}
     = \beta-\cE_t.
\end{align*}
Since $\cE_t=\nrm*{\byh_t}_2-\tri*{\byh_t,\ftil_t}\geq{}0$, this term---up to a small additive error $\beta$---has
a negative contribution to the regret, which we make use of in a moment.

Now, for the second term we begin by using \Holder's inequality and
AM-GM, which implies that for any $\eta_2>0$,
\begin{align*}
  \tri*{\mu_t-\astar,\bfstar-\byh_t}
  \leq{} \frac{1}{2\eta_2}\nrm*{\mu_t-\astar_t}_{\Sigma_t^{-1}}^{2} + \frac{\eta_2}{2}\nrm*{\byh_t-\bfstar_t}_{\Sigma_t}^{2}.
\end{align*}
The second term is precisely what we want, and we will show now that
the first term is cancelled by $\cE_1$ if  $\eta_2$ is chosen
appropriately. We use the lower bound on $\Sigma_t$ from \pref{eq:pt_cov}, which implies that
\begin{align*}
  \nrm*{\mu_t-\astar_t}_{\Sigma_t^{-1}}^{2}
  \leq{} \frac{\infdim{}}{\alpha_t}\nrm*{\mu_t-\astar_t}_{2}^{2}
  &= \frac{\infdim{}}{\alpha_t}\prn*{\nrm*{\astar_t}_2^{2} + \nrm*{\mu_t}_2^{2} 
  -2\tri*{\mu_t,\astar_t}
    }.
    \intertext{Using the value of $\mu_t$ from \pref{eq:pt_cov}, this
    is equal to}
  &= \frac{\infdim{}}{\alpha_t}\prn*{1 + (1-\alpha_t)^{2}
    -2(1-\alpha_t)\tri*{\ytil_t,\ftil_t}
    } \\
  &= \frac{\infdim{}(1-\alpha_t)}{\alpha_t}\prn*{(1-\alpha_t)^{-1} + (1-\alpha_t)
    -2\tri*{\ytil_t,\ftil_t}
    }.
\end{align*}
An elementary calculation reveals that since
    $\alpha_t\leq{}1/2$, 
    \[
      (1-\alpha_t)^{-1} + (1-\alpha_t) =
    \frac{1+(1-\alpha_t)^{2}}{1-\alpha_t} = 2 +
    \frac{\alpha_t^{2}}{1-\alpha_t} \leq{} 2+2\alpha_t^{2}.
  \]
  Hence, we have
  \begin{align*}
    \frac{\infdim{}(1-\alpha_t)}{\alpha_t}\prn*{(1-\alpha_t)^{-1} + (1-\alpha_t)
    -2\tri*{\ytil_t,\ftil_t}
    }
    &\leq{} \frac{\infdim{}(1-\alpha_t)}{\alpha_t}\prn*{2 - 2\tri*{\ytil_t,\ftil_t} +
      2\alpha_t^{2}} \\
    &= \frac{2\infdim{}(1-\alpha_t)}{\beta}\prn*{\nrm*{\byh_t} -
      \tri*{\byh_t,\ftil_t} } + 2\infdim{}\alpha_t\\
    &= \frac{2\infdim{}(1-\alpha_t)}{\beta}\cE_t + 2\infdim{}\alpha_t\\
    &\leq{} \frac{2\infdim{}}{\beta}\cE_t + 2\infdim{}\alpha_t.
  \end{align*}
Combined with all of the calculations so far, this gives
\begin{align*}
  \max_{\astar\in\Aball}\En_{a_t\sim{}p_t}\brk*{\tri*{\bfstar_t,a_t}-\tri*{\bfstar_t,\astar}}
&\leq{} \beta - \cE_t + \frac{1}{2\eta_2}\prn*{\frac{2\infdim{}}{\beta}\cE_t +
                                                                                                2\infdim{}\alpha_t}
                                                                                                + \frac{\eta_2}{2}\nrm*{\byh_t-\bfstar_t}_{\Sigma_t}^{2}. 
\end{align*}
By choosing $\eta_2 = \frac{\infdim{}}{\beta}$, the $\cE_t$ terms cancel, and
we are left with
\begin{equation}
  \max_{\astar\in\Aball}\En_{a_t\sim{}p_t}\brk*{\tri*{\bfstar_t,a_t}-\tri*{\bfstar_t,\astar}}
\leq{} \beta(1+\alpha_t) +
                                                                                                \frac{\infdim{}}{2\beta}\nrm*{\byh_t-\bfstar_t}_{\Sigma_t}^{2}
  \leq{} \frac{3}{2}\beta + \frac{\infdim{}}{2\beta}\nrm*{\byh_t-\bfstar_t}_{\Sigma_t}^{2}.\label{eq:cont_case2}
\end{equation}

\paragraph{Final bound.}
 Combining equations \pref{eq:cont_case1} and \pref{eq:cont_case2}, we
 are guaranteed that in every round, regardless of which case holds,
 \[
   \max_{\astar\in\Aball}\En_{a_t\sim{}p_t}\brk*{\tri*{\bfstar_t,a_t}-\tri*{\bfstar_t,\astar}}
   \leq{} 9\beta + \frac{4\infdim{}}{\beta}\nrm*{\byh_t-\bfstar_t}_{\Sigma_t}^{2}.
 \]
 Using \pref{eq:regret_generic}, this implies that
 \begin{align*}
   \Reg \leq{} \frac{8\infdim{}}{\beta}(\RegSquare + 8\log(\delta^{-1})) + 9\beta{}T + \sqrt{2T\log(\delta^{-1})}.
 \end{align*}
Hence, by choosing $\beta =
\sqrt{\frac{\infdim{}(\RegSquare+8\log(\delta^{-1}))}{T}}$, we have
\[
  \Reg \leq{} 18\sqrt{\infdim{}T\RegSquare} + 90\sqrt{\infdim{}T\log(\delta^{-1})}.
\]

\end{proof}


\end{document}